\documentclass{article}
\newif\ifconference

\conferencefalse

\newcommand{\figura}[2][]{%
  \ifconference
    \includesvg[#1]{#2}%
  \else
    \includegraphics[#1]{#2.png}%
  \fi
}
\usepackage[preprint]{neurips_2026}

\usepackage[utf8]{inputenc} 
\usepackage[T1]{fontenc}    
\usepackage{microtype}      

\usepackage{amsmath, amssymb, amsfonts, amsthm}
\usepackage{nicefrac}       
\usepackage{pifont}         

\usepackage{geometry}
\usepackage{booktabs}       
\usepackage{multirow}
\usepackage{makecell}
\usepackage{array}
\usepackage{longtable}
\usepackage{colortbl}
\usepackage{enumitem}

\usepackage[table]{xcolor}  
\usepackage{graphicx}
\usepackage{subcaption}
\usepackage{tikz}
\usepackage{svg}
\usepackage{wrapfig}

\usepackage{url}            
\usepackage{environ}
\usepackage{xspace}

\usepackage{natbib}
\usepackage{xr-hyper}
\usepackage{hyperref}       
\usepackage{cleveref}
\definecolor{myorange}{RGB}{230,120,20}
\colorlet{summaryyellow}{yellow!15}
\colorlet{oursummary}{summaryyellow!85!gray}
\definecolor{apexblue}{HTML}{378ADD}
\definecolor{apexred}{HTML}{E24B4A}
\definecolor{apexgreen}{HTML}{27AE60}
\definecolor{apexpurple}{HTML}{7F77DD}
\definecolor{novelty}{HTML}{1A5276}

\newcommand{\xmark}{\textcolor{red!50!black}{\ding{55}}}
\let\oldcheckmark\checkmark
\renewcommand{\checkmark}{\textcolor{green!50!black}{\oldcheckmark}}

\newcommand{\Dd}{\mathcal{D}}
\newcommand{\Tt}{\mathcal{T}}
\newcommand{\ours}{\textsc{APEX}\xspace}

\theoremstyle{plain} 
\newtheorem{theorem}{Theorem}

\theoremstyle{definition} 
\newtheorem{definition}{Definition}

\theoremstyle{remark} 
\newtheorem{remark}{Remark}

\graphicspath{{./figures/}}
\begin{document}

\title{\Large \ours: Assumption-free Projection-based Embedding eXamination Metric for Generative Distribution Evaluation}
\author{%
  Caterina Gallegati$^{1}$ \quad Monica Bianchini$^{1}$ \quad Franco Scarselli$^{1}$ \\
  $^{1}$University of Siena, Italy \\
  \texttt{c.gallegati@student.unisi.it,\{monica.bianchini, franco.scarselli\}@unisi.it} \\
  \AND
  Vittorio Murino$^{2,3}$ \quad Barbara Toniella Corradini$^{2}$ \\
  $^{2}$AI for Good (AIGO), Istituto Italiano di Tecnologia, Italy \\
  $^{3}$University of Verona, Italy \\
  \texttt{\{barbara.corradini, vittorio.murino\}@iit.it} \\
}

\maketitle

\begin{abstract}
As generative models achieve unprecedented visual quality, the gold standard for image evaluation remains traditional feature-distribution metrics.
However, these metrics are provably hindered by the \textit{closed-vocabulary bottleneck} of outdated features and the \textit{assumptive bias} of rigid parametric formulations.
Recent alternatives exploit modern backbones to solve the feature bottleneck, yet continue to suffer from parametric limitations. 
To close this gap, we introduce APEX (Assumption-free Projection-based Embedding eXamination), a novel evaluation framework leveraging two open-vocabulary foundation models (CLIP and DINOv2) as feature extractors and the Sliced Wasserstein Distance (SWD) as similarity measure. Since SWD requires no assumptions about the compared data distributions, APEX operates as an assumption-free metric.
Moreover, we prove with theoretical and empirical evidences that it inherits effective scalability to high-dimensional spaces. 
Benchmarking APEX metrics against established baselines, we show that they exhibit strong robustness to visual degradations, intra- and cross-dataset stability, ensuring highly stable evaluations on out-of-domain datasets. 
\end{abstract}

\section{Introduction} 
\label{sec:intro}
In recent years, the emergence of novel generative paradigms has fundamentally revolutionized the landscape of visual synthesis. 
Specifically, the transition from traditional adversarial frameworks to diffusion-based and score-matching models has yielded unprecedented advancements in the quality, diversity, and realism of generated images, while enabling highly controllable, large-scale text-to-image synthesis~\citep{rombach2022high, ramesh2022hierarchical, peebles2023scalable, dhariwal2021diffusion, yang2023diffusion, goodfellow2014generative, karras2019style,sohl2015deep, ho2020denoising,song2019generative, song2021score}.\\
From early frameworks to contemporary architectures, the evolution of generative models has been fundamentally driven by the need to quantify synthetic image quality to maximize and refine training outcomes.  
First efforts in image quality evaluation operated exclusively within the image space, typically measuring pixel-level errors or local structural distortions against a ground-truth reference~\citep{wang2009mean, tanchenko2014visual, wang2004image}. 
However, since these metrics operate strictly on low-level visual statistics, they fail to capture high-level semantic realism and human perceptual judgments.
Since the remarkable qualitative progress of image generators relies on 
modelling high-dimensional, complex data distributions~\citep{TheisOB15}, quantitative evaluation of generated samples has shifted from pixel-level comparison to measuring distributional similarity with respect to the target. 
Since deep neural networks inherently capture human-aligned perceptual similarities~\citep{zhang2018unreasonable}, modern evaluation protocols typically project images into a deep feature space and measure the distributional distance between real and synthetic data.
Prominent metrics like Fréchet Inception Distance (FID)~\citep{heusel2017gans} and Kernel Inception Distance (KID)~\citep{binkowski2018demystifying} project both real and synthetic samples into a shared Inception-v3~\citep{salimans2016improved} space before computing their statistical divergence.
Despite their popularity, these metrics present significant limitations, both in terms of the underlying statistical hypotheses and the expressiveness of embedding networks.
On the statistical front, the reliance on rigid formulations introduces an \emph{assumptive bias}: the Gaussian assumption inherent in FID---empirically proven to be incorrect~\citep{jayasumana2024rethinking}---roughly reflects the high-dimensional, multimodal geometry of modern generative latent spaces, while the MMD formulation in KID suffers from extreme sensitivity to kernel hyperparameters~\citep{bischoff2024practical}.
The assumptive bias is mitigated by the Sliced Wasserstein Distance (SWD)~\citep{lee2019sliced}, which does not require unrealistic properties on the compared data distributions or hyperparameters tuning and offers efficiency and sensibility to class-distributional variations~\citep{kwon2026evaluating} despite its poorly intuitive interpretation~\citep{bischoff2024practical}.
In terms of the embedding network, the reliance on Inception-v3 introduces a \emph{closed-vocabulary bottleneck}: being trained exclusively on ImageNet~\citep{deng2009imagenet}, these metrics are blind to the open-world semantics and complex text-alignment of recent multimodal models.
To overcome the closed-vocabulary bottleneck, recent metric proposals have replaced Inception-v3 with large foundation models (e.g., CLIP~\citep{radford2021learning}, DINOv2~\citep{oquab2024dinov}) and dataset-domain agnostic representations (e.g., Wavelet Packet Transform~\citep{veeramacheneni2025frechet}), positioning them as more robust features~\citep{jayasumana2024rethinking, kwon2026evaluating, stein2023exposing, veeramacheneni2025frechet}.
Nevertheless, these alternatives have later been shown to suffer from critical drawbacks as well, such as the strong assumptions behind formulations, the sensitivity to hyperparameter configuration, weakness under non-trivial image perturbations, and poor correlation with human perception \citep{bischoff2024practical}.
These inconsistencies, alongside the unprecedented fidelity of modern synthetic images, raise the question of whether current gold-standards are still sufficient to evaluate the latest models.
\\
In this work we try to address this concern and introduce APEX, an Assumption-free Projection-based Embedding eXamination framework for to evaluate generative distribution.
To overcome the assumptive bias, APEX leverages the SWD, a similarity measure that does not rely on rigid parametric hypotheses or kernel configurations. 
Built upon the one-dimensional projection-based nature of the SWD, APEX also provides a stable and sample-efficient evaluation method even in high-dimensional embedding spaces.
Simultaneously, to resolve the closed-vocabulary bottleneck, our framework operates as a highly flexible, embedding-agnostic tool. 
While we primarily showcase the application to CLIP and DINOv2 as state-of-the-art foundation models, it natively supports any modern feature space. 
Our contributions can be summarized as follows:
\begin{itemize}[leftmargin=*,noitemsep,topsep=0em]
\item[--] We audit failure modes and limitations of existing metrics, analysing how they are affected by backbone choices, distributional assumptions, kernel configurations, finite-sample effects, and cross-domain shifts. 
\item[--] We introduce APEX, a projection-based evaluation framework for comparing image distributions in foundation model embedding spaces.
By combining SWD with CLIP and DINOv2 representations, APEX reduces reliance on rigid parametric assumptions and kernel hyperparameter choices, while remaining flexible to different visual backbones.
\item[--] We investigate the number of random projections sufficient for an accurate SWD estimation, for which we
provide both theoretical motivations and empirical ablations.
\item[--] We benchmark APEX against established metrics via an extensive evaluation protocol spanning heterogeneous visual domains and degradations, analysing finite-sample stability, cross-dataset consistency, and sensitivity to synthetic images, progressive generation and refinement.
\item[--] We verify human alignment of the APEX metrics via a dedicated human perceptual study and a human perceptual benchmark of generated samples.
\end{itemize}
The remainder of the paper is organized as follows.
\Cref{sec:related_works} summarises the established metrics for image evaluation, pointing out their motivations, strengths and limitations. 
\Cref{sec:method} describes the proposed \ours method, providing theoretical and empirical insights on the number of projections used in the SWD estimation. 
\Cref{sec:experiments} presents the evaluation framework used to benchmark the APEX metrics against the existing baselines. 
Results are discussed in~\Cref{sec:discussion}.

\section{Related Work} \label{sec:related_works}
Synthetic image evaluation has undergone a continuous evolution, which can be traced along two interconnected trajectories: the underlying data representation space and the mathematical notion of similarity applied to the resulting distributions.
The Inception Score (IS) was proposed in~\cite{salimans2016improved} as a deep feature-based metric, serving as an alternative to human annotators for assessing the visual quality of synthetic samples. 
Generated images are encoded via an Inception-v3 deep neural network~\cite{szegedy2016rethinking} pre-trained on ImageNet, and the IS score is obtained by computing the entropy of the class probabilities. 
Critically, IS does not involve the target dataset used to train the generative model and the evaluation is limited to the knowledge of the Inception network.
The FID score~\cite{heusel2017gans} was formulated as an improvement over IS. 
Instead of just looking at the final class predictions, FID extracts visual features from an intermediate layer of Inception-v3 for both the target and the generated images. 
Then, it compares the resulting means and covariances by computing the Fréchet distance, assuming that the embedded feature representations follow a Gaussian distribution. 
KID by~\cite{binkowski2018demystifying} aims to mitigate this strong and unrealistic assumption, hence computes the squared MMD between Inception-v3 embeddings via a polynomial kernel. 
Complementary to the aforementioned distribution-based measures, LPIPS~\cite{zhang2018unreasonable} calculates the pixel-level correspondence between paired samples and is widely adopted to quantify the perceptual distance between image pairs. 
By leveraging deep convolutional features, LPIPS aligns comparisons with human judgment. 
Nevertheless, all these metrics benchmark synthetic samples by relying on outdated and closed-vocabulary feature spaces, hence they may not generalize beyond their original training data.
More recently, CLIP-MMD (CMMD)~\cite{jayasumana2024rethinking} has been introduced as a more up-to-date and fair alternative. 
While maintaining MMD as a solution to the incorrect Gaussian assumption, it leverages CLIP as pre-trained embedding network to extract image features. 
This fundamental innovation represents a significant step forward in the evaluation of modern generative models, as it finally breaks free from the long-standing dependence on Inception-v3 and its ImageNet-centric priors.
\citet{bischoff2024practical} present a theoretical study on statistical properties of the most common similarity metrics. 
They criticise MMD for being overly sensitive to kernel selection and hyperparameter configuration. 
Conversely, they emphasise the SWD for its efficiency and scalability, establishing it as an objective and hyperparameter-agnostic baseline for general distribution comparisons.
\citet{kwon2026evaluating} analyses the sensitivity to class-distributional changes of various metrics by leveraging CLIP-based feature embeddings. 
Specifically, both SWD and MMD have been proven to be sensitive to small distributional differences; however, the latter is competitive only for suitable choices of its kernel bandwidth.
As a recent example, the Fréchet Wavelet Distance (FWD)~\citep{veeramacheneni2025frechet} bypasses network biases via parameter-free representations with wavelets coefficients, yet it remains constrained by the unrealistic Gaussian assumptions inherent to the Fréchet metric. 

\section{The APEX metrics} \label{sec:method}
Previous approaches are fundamentally constrained by 
obsolete representation spaces, such as the ImageNet-centric priors of Inception-v3. 
Moreover, they force unrealistic assumptions on data distributions or remain overly sensitive to kernel choice and bandwidth tuning.
We overcome these limitations by proposing
APEX, an Assumption-free Projection-based Embedding eXamination framework for generative distribution evaluation.
APEX is a feature-distribution metric that 
computes an SWD estimation to compare distributions of 
image embeddings extracted by state-of-the-art foundation models, 
CLIP and DINOv2. \\
The \textbf{Sliced Wasserstein Distance}~\cite{rabin2011wasserstein}
aggregates the $2$-Wasserstein distances between one-dimensional projections of two probability distributions, $\mu$ and $\nu$, by integrating over all directions on the unit sphere.
Let $\pi_\theta: \mathbb{R}^d \rightarrow \mathbb{R}$ be the projection 
onto a vector $\theta \in \mathbb{S}^{d-1}$ of the 
$d$-dimensional unit sphere, then the squared SWD is defined as
\begin{align}
\label{eq:sliced_wass_integral}
    &SW_2(\mu,\nu)^2=
    \int_{\mathbb{S}^{d-1}} \mathcal{W}_2(\pi_{\theta_\sharp}\mu,\pi_{\theta_\sharp}\nu)^2 d\theta, \\ &\mathcal{W}_2(\pi_{\theta_\sharp}\mu,\pi_{\theta_\sharp}\nu)^2= 
    \min_{\gamma \in \Gamma(\pi_{\theta_\sharp}\mu,\pi_{\theta_\sharp}\nu)} \int |u-v|^2 d\gamma(u,v) \nonumber.   
\end{align}
In practice, the SWD can be approximated by using a simple Monte Carlo scheme that uniformly draws $L$ sample directions ${\theta_l}$ on $\mathbb{S}^{d-1}$, and replaces the integration with a finite-sample average.
The Monte Carlo approximation effectively makes the SWD both easy to compute in practice and computationally efficient.
Moreover, no additional careful hyperparameter tuning is required. 
The estimated SWD in APEX computes the distance between empirical distributions of image embeddings, extracted from CLIP or DINOv2.\\
\textbf{CLIP} (Contrastive Language--Image Pretraining) by~\cite{radford2021learning} jointly trains an image encoder and a text encoder on $400$M image--caption pairs scraped from the web.
The encoders are optimized with a contrastive objective that pulls matching image--text pairs together and pushes mismatched pairs apart in a shared embedding space.
At convergence, the image encoder produces representations that capture high-level semantic content aligned with natural language descriptions, enabling strong zero-shot transfer across visual tasks. \\
\textbf{DINOv2} (knowledge DIstillation with NO labels) by~\cite{oquab2024dinov} trains a Vision Transformer with a self-supervised objective combining image- and patch-level self-distillation.
A student network is trained to match the output of a momentum-updated teacher on global \texttt{[CLS]} and local patch tokens, with no text supervision or labels.
Depending on whether CLIP or DINOv2 embeddings are utilized, we denote the resulting metric as \textbf{\ours-CLIP} or \textbf{\ours-DINO}, respectively.

\subsection{Analysis on the number of projections} \label{subsec:num_directions}
When computing the
Monte Carlo approximation of the 
SWD, 
minimising the number of projections reduces the overall computational demand of the implementation.
Additionally, it prevents the estimator to excessively favour high-frequency noise in the sample distributions. 
In the following result, we investigate how the structure of (a finite number of) image embeddings influences the number of projections $L$ for a stable Monte Carlo approximation of squared SWD. \begin{theorem} \label{theorem}
Let $\mu$ and $\nu$ be data distributions supported on a bounded $k$-dimensional manifold $\mathcal {M} \subseteq\mathbb{R}^d$ with diameter $D>0$. 
Assume there is a constant $A\geq 1$ such that, for every $0 < \epsilon \leq D$, the minimum cardinality of an
$\epsilon$-net with centers in $\mathcal M$ is $\mathcal N(\mathcal M,\epsilon)\leq\left(\frac{AD}{\epsilon}\right)^k$, using the ambient Euclidean distance.
Let $SW_2(\mu_n,\nu_n)^2$ be the squared Sliced Wasserstein distance computed on the empirical distributions $\mu_n, \nu_n$ induced by $n$ samples, and $\widehat{SW_2}(\mu_n,\nu_n)^2$ its Monte Carlo estimate using $L$ uniformly sampled independent projections.
    Then, for any $\xi>0$ and $0<\delta<1$, if 
    \begin{equation} \label{eq:bound}
        L \geq \frac{2D^4}{\xi^2} \left[ 2nk \log \left( \frac{16AD^2}{\xi} \right) + \log\left( \frac{2}{\delta} \right) \right]
    \end{equation}
    the error $|\widehat{SW_2}(\mu_n,\nu_n)^2-SW_2(\mu_n,\nu_n)^2|\leq \xi$ with probability at least $1-\delta$.
\end{theorem}

The proof of~\Cref{theorem} is reported in Appendix \ref{app:new_thm}.
The result establishes that a single collection of random projections accurately approximates all comparisons between empirical distributions supported on the embedding manifold, with a sufficient projection count controlled by the intrinsic geometric complexity.
In particular, the number of projections scales linearly with $k$, the intrinsic dimension, 
while being independent on the ambient dimension $d$.
We discuss how the assumptions can be interpreted in our experimental setting in Appendix \ref{app:new_thm} in more details.\\
In practical scenarios, Eq.\eqref{eq:bound} cannot be instantiated numerically, since quantities related to the geometric complexity are generally unknown. 
However, it is still possible to empirically derive $L$ by analysing how the metrics evolve as the number of projections increases. 
We provide evidences on this result with the ablation study in \Cref{app:fig:swd_projections} in Appendix~\ref{ablation_nproj}. 
We show the trade-off between the number of projections $L$ used in the estimate and \ours metrics score across datasets and degradations, together with
runtime as $L$ increases.
\section{Experiments} \label{sec:experiments}
We design a comprehensive evaluation framework to benchmark the 
\ours{} metrics against established baseline metrics.
We describe the six benchmark datasets in~\Cref{subsec:datasets} and the evaluated metrics (including APEX and the baselines) in~\Cref{subsec:evaluation_metrics}.
In~\Cref{subsec:protocol}, we define a multi-axis evaluation protocol going beyond aggregate benchmark scores. 
Experiments were run on a GPU cluster; each run used a single NVIDIA Tesla V100-SXM2 (16\,GB) and 4--8 cores of an Intel Xeon Silver 4210 CPU (2.2\,GHz).

\subsection{Datasets}
\label{subsec:datasets}
We conduct experiments across multiple datasets characterized by~diverse domains, resolutions, and semantic content (see~\Cref{fig:datasets}).

\textbf{COCO-30k} is composed of $30{,}000$ image--caption pairs from the MS-COCO 2014 validation split~\citep{lin2014microsoft}. 
Unlike object-centric samplings, it provides unconstrained representations of 80 distinct object categories in real-world contexts. 
Images retain their original aspect ratios with the longest edge typically resized to 640 pixels.

\textbf{HAM10000}~\citep{tschandl2018ham10000} comprises $10{,}015$ $600{\times}450$ pixels multi-source dermoscopic images of pigmented skin lesions. 
The dataset, spanning seven distinct diagnostic categories (including melanoma and basal cell carcinoma), provides a benchmark for fine-grained classification and robustness in real-world clinical domains.

\textbf{CelebA-HQ}~\citep{karras2018progressive} is a refined version of the large-scale CelebFaces Attributes (CelebA) dataset (\citet{liu2015deep}), containing $30{,}000$ diverse facial images.
The images were curated through a rigorous pipeline involving artifact removal, super-resolution, and standardised facial alignment, yielding high-fidelity data at a resolution of $1024{\times}1024$ pixels.

\textbf{NIH ChestX-ray14} by~\citet{chestxray2017wang} comprises $112{,}120$ frontal-view grayscale radiographs from $30{,}805$ patients. 
The dataset encompasses $14$ distinct thoracic pathologies where multiple conditions often co-occur. 
Image resolution is $1024{\times}1024$ pixels.

\textbf{NWPU-RESISC45}~\citep{cheng2017nwpu} is composed of $31{,}500$ remote sensing images uniformly distributed across $45$ diverse scene classes (e.g., airports, industrial areas, and wetlands).
The dataset contains $256{\times}256$ pixels images with extreme variability in spatial resolution (from $0.2$ to $30$ meters per pixel).

\textbf{AGIQA-3K} by~\citet{li2023agiqa} comprises $2{,}982$ AI-generated images annotated with multi-dimensional human subjective ratings. 
Including images by ten modern generators (including Stable Diffusion, Midjourney, and DALL-E 2), it allows to evaluate perceptual image quality and text-to-image alignment, providing a rigorous benchmark for assessing generative artifacts across diverse styles and prompts.
\begin{figure*}[htbp]
    \centering
    \captionsetup[subfigure]{font=footnotesize}
    \begin{subfigure}[t]{0.16\textwidth}
        \centering
        \figura[width=\linewidth]{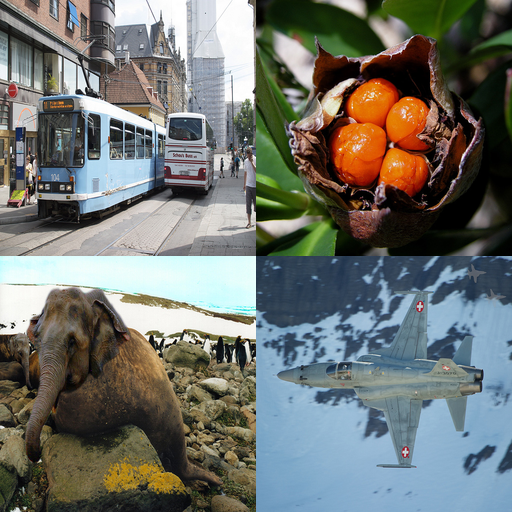}
        \caption{\tiny COCO-30k}
    \end{subfigure}%
    \hfill
    \begin{subfigure}[t]{0.16\textwidth}
        \centering
        \figura[width=\linewidth]{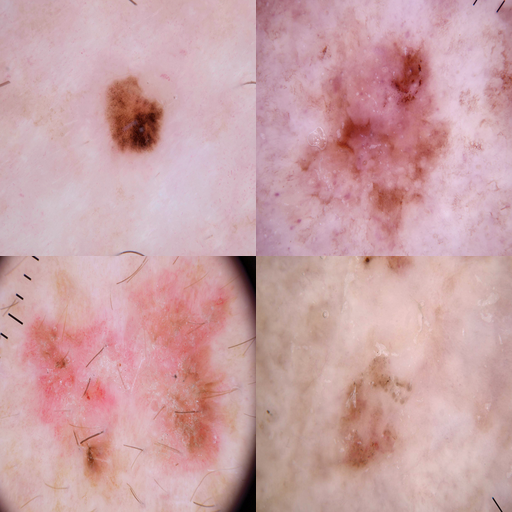}
        \caption{\tiny HAM 10000}
    \end{subfigure}%
    \hfill
    \begin{subfigure}[t]{0.16\textwidth}
        \centering
        \figura[width=\linewidth]{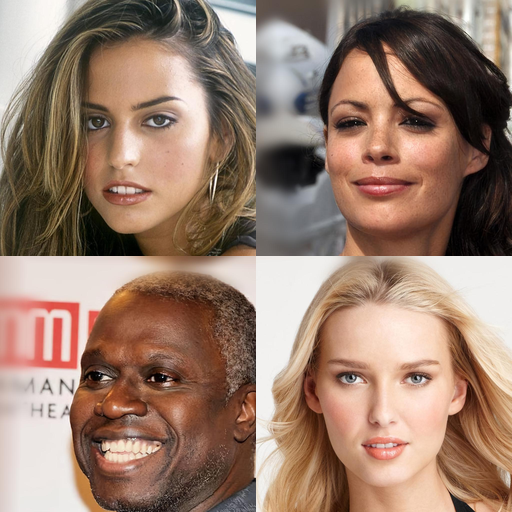}
        \caption{\tiny CelebA-HQ}
    \end{subfigure}%
    \hfill
    \begin{subfigure}[t]{0.16\textwidth}
        \centering
        \figura[width=\linewidth]{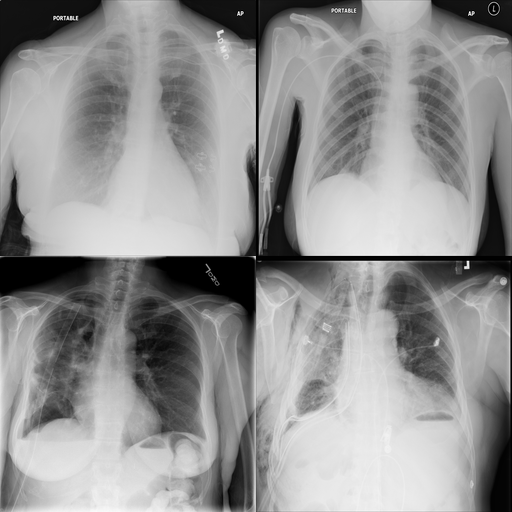}
        \caption{\tiny NIH \mbox{ChestX-ray14}}
    \end{subfigure}%
    \hfill
    \begin{subfigure}[t]{0.16\textwidth}
        \centering
        \figura[width=\linewidth]{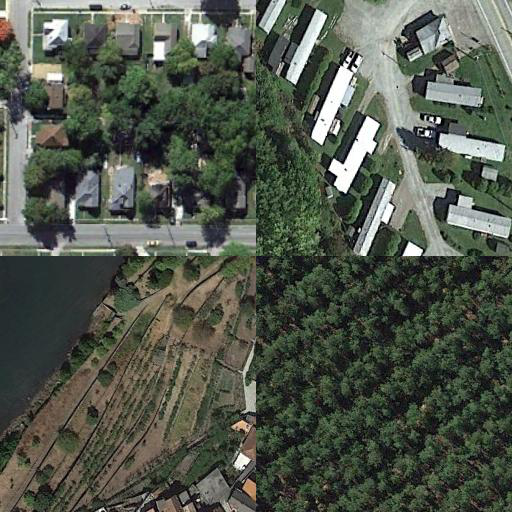}
        \caption{\tiny NWPU-RESISC45}
    \end{subfigure}%
    \hfill
    \begin{subfigure}[t]{0.16\textwidth}
        \centering
        \figura[width=\linewidth]{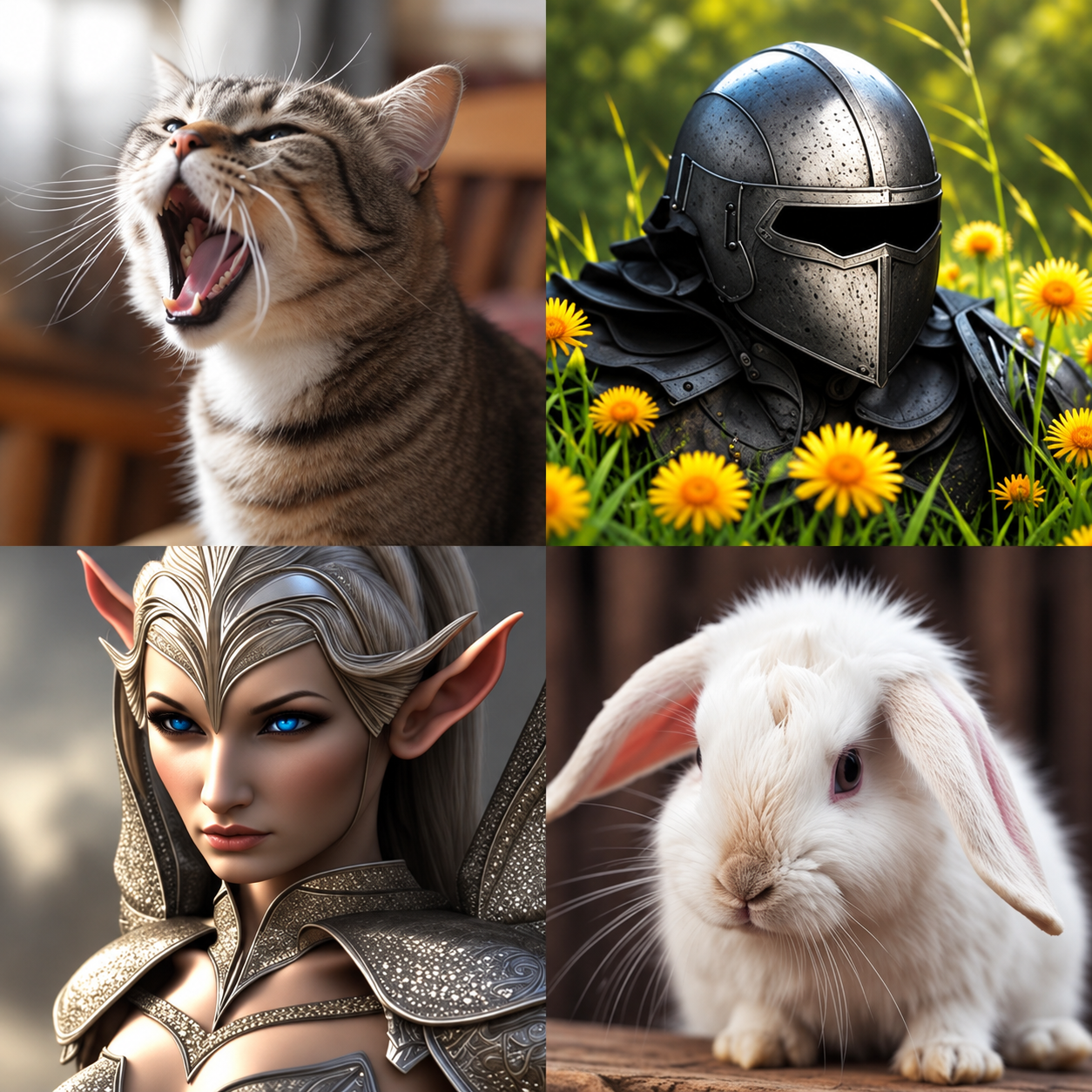}
        \caption{\tiny AGIQA-3k}
    \end{subfigure}%
    \caption{\textbf{Sample images from the six evaluation datasets}. 
    We span different domains: 
    natural scenes (COCO-30k), 
    dermoscopy (HAM10000), 
    faces (CelebA-HQ), 
    medical radiographs (NIH Chest X-Ray), 
    satellite (NWPU-RESISC45) and synthetic imagery (AGIQA-3k).}
    \label{fig:datasets}
\end{figure*}%

\subsection{Evaluation Metrics} \label{subsec:evaluation_metrics}
To comprehensively assess evaluation robustness under distribution shifts, we benchmark a diverse suite of similarity metrics. \Cref{tab:metrics-overview} summarizes the technical specifications of the evaluated metrics.

\textbf{Proposed Metrics (\ours).} 
We evaluate the two metrics formulated in~\Cref{sec:method}. 
\ours-CLIP leverages OpenAI's CLIP ViT-L/14@336px 
yielding 768-dimensional \texttt{[CLS]} embeddings. 
Conversely, \ours-DINO captures a richer architectural hierarchy by concatenating the 1024-dimensional \texttt{[CLS]} tokens from shallow (6), middle (12) and deeper (23) layers of DINOv2 ViT-L/14 into a single 3,072-dimensional representation.
To prevent the activation magnitudes of deeper layers from dominating the joint space, each layer is $\ell_2$-normalized prior to concatenation.
This multi-scale strategy leverages the hierarchical representations of the backbone, capturing low-level textures, mid-level structures and high-level semantics.

\textbf{Baselines.}
\begin{table*}[htpb]
\centering
\caption{~\textbf{Overview of the evaluation metrics 
and their structural properties.}
\textbf{AF}~(Assumption-Free 
feature distributions),
\textbf{FE}~(Foundation model Embeddings),
\textbf{OT}~(Optimal Transport 
distance).} 
\label{tab:metrics-overview}
\resizebox{0.7\textwidth}{!}{\begin{minipage}{\textwidth}
\setlength{\tabcolsep}{5pt}
\renewcommand{\arraystretch}{1.15}
\footnotesize
\begin{tabular}{@{}lllcccc@{}}
\toprule
\textbf{Metric}
  & \textbf{Backbone}
  & \textbf{Distance}
  & \textbf{Dim.}
  & \textbf{AF}
  & \textbf{FE}
  & \textbf{OT} \\
\midrule
FID~\citep{heusel2017gans}
  & Inception-v3   & Fr\'{e}chet                & 2048
  & \xmark & \xmark & \xmark \\
KID~\citep{binkowski2018demystifying}
  & Inception-v3   & MMD (polynomial)           & 2048
  & \checkmark & \xmark & \xmark \\
CMMD~\citep{jayasumana2024rethinking}
  & CLIP ViT-L/14  & MMD (RBF, $\sigma\!=\!10$) & 768
  & \checkmark & \checkmark & \xmark \\
DINO+MMD\textsuperscript{$\dagger$}
  & DINOv2 ViT-L/14 & MMD (median heur.)        & 3072
  & \checkmark & \checkmark & \xmark \\
  FWD\citep{veeramacheneni2025frechet}
  & --- & Wavelet Packet Transform        & ---
  & \xmark & \xmark & \xmark \\
LPIPS\textsuperscript{$\ast$}~\citep{zhang2018unreasonable}
  & AlexNet         & Perceptual loss            & ---
  & \checkmark & \xmark & \xmark \\
\rowcolor{gray!10}
\ours-CLIP
  & CLIP ViT-L/14   & SWD ($L$ finite) & 768
  & \checkmark & \checkmark & \checkmark \\
\rowcolor{gray!10}
\ours-DINO
  & DINOv2 ViT-L/14 & SWD ($L$ finite) & 3072
  & \checkmark & \checkmark & \checkmark \\
\bottomrule
\end{tabular}
\end{minipage}
}

\vspace{2pt}
{\footnotesize
\textsuperscript{$\ast$}Paired metric (pixel-level correspondence required). \;
\textsuperscript{$\dagger$}Exploratory variant evaluated in this work.
}
\end{table*}
We benchmark our approaches against established literature standards, computing CMMD and FWD, 
alongside FID, KID and LPIPS as previous generative standards. 
Additionally, to rigorously disentangle the contribution of the feature backbone from the choice of statistical distance, we investigate DINO+MMD as an exploratory variant. 
This baseline applies the MMD to the $3{,}072$-dimensional DINOv2 feature space, with a mixture of $7$ RBF bandwidths relative to the median pairwise distance.
\subsection{Evaluation Protocol}
\label{subsec:protocol}
Our evaluation framework is designed to compare the metrics based on: responsiveness to diverse pixel- and latent-level distortions (\Cref{subsubsec:degradation_sensitivity}),
intra-dataset stability, efficiency under finite sample sizes, invariance of the degradation signal across different domains (\Cref{sec:substability}), capability to capture subtle details during the generative refinement phase (\Cref{subsubsec:refine_sens}), and overall alignment with human perceptual judgments (\Cref{subsubsec:human}).

\subsubsection{Degradation Sensitivity}\label{subsubsec:degradation_sensitivity}
To evaluate APEX under controlled distribution shifts, we apply both pixel- and latent-space perturbations across domains (examples in Appendix~\ref{app:fig:degradations}).
Pixel-level degradations include Gaussian noise, blur, resolution drop, JPEG compression, and colour shifts, covering sensor noise, defocus, information loss, compression artifacts, and illumination changes.
In latent space, we perturb Stable Diffusion VAE features to probe the robustness of learned representations.

\begin{definition}[Degradation signal]
Let $\Dd_k$ be a dataset, $\tau$ a degradation applied with severity $s$, and $\phi$ a feature extractor. Then, for some notion of distance $d$, the degradation signal is defined as
\begin{equation} \label{eq:degrad_sens}
    \Delta_k(\tau, s) \;=\; d\bigl(\phi(\tau_s(\Dd_k)),\;
    \phi(\Dd_k)\bigr)
\end{equation}
where $\tau_s(\Dd_k)$ denotes the element-wise application of $\tau_s$ to samples in $\Dd_k$.\\
A well-behaved metric should satisfy \emph{monotonicity} relative to the degradation signal as severity increases:
$\Delta_k(\tau, s_1) < \Delta_k(\tau, s_2)$ whenever $s_1 < s_2$,
i.e.\ stronger perturbations yield larger distances.
\end{definition}

\subsubsection{Intra- and Cross-dataset Stability}
\label{sec:substability}
Distributional metrics rely on finite samples, therefore suffer from inherent estimation noise that sets a baseline floor, preventing reliable detection of minor degradations. 
Additionally, achieving a stable measurement often demands a prohibitively large number of images, and cross-domain comparisons are fundamentally limited by the absence of an absolute scale. 
To evaluate these vulnerabilities, we systematically characterise for each metric: intra-dataset finite-sample bias, convergence against evaluation dataset size, and cross-dataset consistency.
\begin{definition}[Finite-sample Bias] 
\label{def:stab_coeff}
Let $A_k,B_j$ be two subsets of equal size sampled from the datasets $\Dd_k$ and $\Dd_j$, respectively. 
Assume to repeat the sampling procedure $R$ times.
Then, for some notion of distance $d$ and feature extractor $\phi$, the \textit{finite-sample bias} is
\begin{equation}
\begin{aligned}
        \bar{\nu}_{k,j} \;=\; \frac{1}{R}\sum_{i=1}^{R} \hat{\nu}_{k,j}^{(i)},
    \quad
    \hat{\nu}_{k,j}^{(i)} \;=\; d\bigl(\phi(A_k^{(i)}),\;
    \phi(B_j^{(i)})\bigr). 
\end{aligned}
\end{equation}
Then, 
the \textit{coefficient of variation} between the standard deviation $\sigma(\cdot)$ of each finite-sample bias and its absolute value is defined as
\begin{equation}
    \mathrm{CV}_{k,j} \;=\;
    \frac{\sigma\!\bigl(\hat{\nu}_{k,j}^{(1)}, \dots
    \hat{\nu}_{k,j}^{(R)}\bigr)}{|\bar{\nu}_{k,j}|}.
\end{equation}
\end{definition}
\begin{definition}[
Cross-dataset Consistency]
Let $\tau \in \Tt$ be a degradation applied with different severity $s \in \mathcal{S}$ to some dataset $\Dd_k$ and define the degradation signal $\Delta_k(\tau, s)$ as in \Cref{eq:degrad_sens}. 
Then, given ($\Dd_j,\Dd_k$) a pair of distinct datasets, the log-ratio stability across domains is defined as
\begin{equation}
    L_{jk}(\tau, s) = \log_2 \frac{\Delta_j(\tau, s)}{\Delta_k(\tau, s)}.
\end{equation}
The \emph{cross-dataset consistency} is the mean absolute log-ratio across degradations and datasets:
\begin{equation}
    \Lambda(d) = \frac{1}{M} \sum_{\tau, s, j < k} |L_{jk}(\tau, s)|,
\end{equation}
where $M = |\Tt| \cdot |\mathcal{S}| \cdot \binom{K}{2}$, and $d$ is the distance in $\Delta_k(\tau,s)$. Lower values of $\Lambda$ indicate more consistent metric responses across domains.
Additionally, we distinguish correctable from uncorrectable biases by estimating the \textit{interaction variance}: 
$V_{jk}=\text{Var}_{\tau,s}[L_{j,k}(\tau,s)]$.
We summarise it over all dataset pairs with the \emph{mean interaction variance}
\begin{equation}
    V_{\text{m}}(d) = \binom{K}{2}^{-1} \sum_{j<k} V_{jk},
\end{equation}
where the variance is the unbiased sample variance over the $|\Tt|\cdot|\mathcal{S}|$ degradation--severity pairs.
Low $V_{jk}$ indicates a mostly multiplicative dataset bias, which could in principle be corrected by normalisation.
High $V_{jk}$ indicates an interaction bias, meaning that the metric response depends jointly on the dataset and the degradation type.
\end{definition}
\textbf{Intra-dataset Finite-sample Bias.} 
We assess how each metric behaves when no distribution shift is present.
For each source dataset $\Dd_k$, we generate $R{=}20$ random splits $A_k^{(i)},B_k^{(i)} {\subseteq} \ \Dd_k,\, i{=}\{1,\dots,R\}$.
Since $A_k$ and $B_k$ are drawn from the same dataset, a reliable metric should return values close to zero and remain consistent across splits.
For each $i$, we compute the finite-sample bias of metric $\bar{\nu}_{k,k}$ and the intra-dataset coefficient of variation $\mathrm{CV}_{k,k}$.
As $\bar{\nu}_{k,k}$ is expressed in the units of each metric, 
it cannot be directly compared across metrics.
We therefore measure how many degradations each metric separates from its own floor on a shared subset of size $n{=}1{,}000$.
On the same split, we evaluate Eq.(\ref{eq:degrad_sens}) and define a degradation as \emph{detected} when the mean $\bar{\Delta}_k(\tau,s)$ exceeds the $95$th percentile of $\{\hat{\nu}_{k,k}^{(i)}\}_{i=1}^{R}$.
Since this test only depends on the ordering of the values, the number of detected degradations is invariant to the scale of the metric and to any monotone transformation.

\textbf{Cross-dataset Consistency.} 
\label{subsubsec:domain_consistency}
To measure how the metrics' response varies across domains and degradations, we extend the cross-dataset evaluation of~\cite{veeramacheneni2025frechet} and the distortion sensitivity analysis of~\cite{jayasumana2024rethinking} into a single comprehensive framework. 
We estimate degradation signal across domains and evaluate each metric over all the degradations and severities. 

\textbf{Sample Efficiency and Scalability.}
We assess metric stability and runtime as sample size increases, reporting results under Gaussian blur corruption at severity level 2 as an illustrative example.
For subsets $A_k \subseteq \Dd_k$ and $B_k \subseteq \tilde{\Dd}_k$, the minimum reliable $N$ is the smallest size yielding stable estimates at acceptable computational cost.
 

\subsubsection{Consistency Across Coarse-to-Fine Generation}
\label{subsubsec:refine_sens}
We evaluate \ours throughout the generative process to assess responsiveness to evolving image fidelity. 
We sample intermediate images from Stable Diffusion v1.5, specifically monitoring the final steps where updates restrict to high-frequency components and localized refinements (\Cref{app:fig:generation_refinement}).
\begin{definition}[Refinement sensitivity]
Given $\Dd$ a dataset of pairs of images and captions, sample a subset of images $\mathcal{I}_{\text{real}} = \{I_{\text{real}}^{(j)}\}_{j=1}^{N}$ and a subset of captions $\mathcal{C} = \{c_i\}_{i=1}^{N}$. 
Let $f$ be a text-conditioned model with $T$ generation steps, and denote the resulting set of synthetic images as
\begin{equation}
        \mathcal{I}_{\text{gen}} =\{I_{\text{gen}}^{(i)}\}_{i=1}^{N} \quad \text{where} \\  I_{\text{gen}}^{(i)}:=f(c_i; T).
\end{equation}
Then, for some notion of distance $d$ and feature extractor $\phi$, the refinement sensitivity is defined as 
\begin{equation}
\Delta_{gen}(t) = d\bigl(\phi({\mathcal{I}_{\text{gen}}(t)}), \phi({\mathcal{I}_{\text{real}}})\bigr) \quad t\in[0,T].
\end{equation}
\end{definition}
We choose $N{=}5{,}000$, $T{=}100$, $t \in \{5, 10, 20, 30, 50, 75,$  $95, 96, 97, 98, 99\}$.
A well-behaved metric should exhibit continuous improvement as $t$ grows, being capable of capturing subtle perceptual enhancements of final stages.

\subsubsection{Human Perceptual Correlation Study} \label{subsubsec:human}

To quantify how \ours aligns with human judgements, we conduct a perceptual study on controlled degradations.
Annotators rate corrupted images against their original references on a 1--5 scale.
Across five datasets, we collect $900$ ratings in $25$ sessions over $10{,}850$ stimuli, covering five pixel-space degradations and latent distortions at five severities.
Ratings are aggregated into a Mean Opinion Score~\citep{streijl2016mean} for each dataset--degradation--severity condition and compared with the corresponding metric values.
Second, we test against ten generative models using AGIQA-3K, which provides quality and text--image alignment MOS: we group the images by generator and by MOS level, compare each group with an equally sized set of real COCO images, and correlate the metric with the group's MOS.
Alignment is assessed using Spearman $\rho$~\citep{spearman_ro} and Kendall $\tau$~\citep{kendall_tau}.
Further experimental details are provided in Appendix~\ref{app:human}.

\section{Results and Discussion} \label{sec:discussion}
\paragraph{Degradation sensitivity.}  
\ours{} consistently displays monotonic sensitivity to progressive perturbations across domains (\Cref{app:fig:sensitivity} Appendix~\ref{app:fig:degradations}), bypassing the domain-sensitive fluctuations of the baselines. 
Furthermore, our per-layer analysis of \ours-DINO (\Cref{app:fig:sensitivity_apexdino}) highlights a hierarchical degradation response: low-level corruptions affect shallow layers, structural changes impact intermediate ones, latent distortions disrupt visual representation uniformly across layers.

\textbf{Intra- and cross-dataset stability.}
\Cref{tab:stability-merged} shows that \ours{} combines high sensitivity with a stable finite-sample floor.
\ours-CLIP detects $96.7\%$ of the degradations, the best rate overall (on par with CMMD) and all of them on AGIQA-3K, while \ours-DINO detects $93.3\%$, clearly ahead of KID ($89.4\%$), FID ($87.2\%$), LPIPS ($66.7\%$) and FWD ($56.7\%$).
At the same time, its self-distance is stable across splits (CV $0.02$--$0.07$): with the same CLIP features, \ours-CLIP matches the detection rate of CMMD with a 2--5$\times$ lower CV ($0.04$--$0.06$ vs.\ $0.09$--$0.25$), and \ours-DINO keeps a positive, stable floor where DINO+MMD and KID fluctuate around zero (CV up to $375$).
Across datasets, \ours-CLIP is the most consistent metric, with the lowest log-ratio deviation ($\Lambda{=}0.87$), followed by FID and \ours-DINO.
Since each \ours{} variant shares its backbone with an MMD baseline, this comparison isolates the distance: replacing MMD with SWD on the same features lowers $\Lambda$ by about $20\%$ for both CLIP ($1.07$ to $0.87$) and DINOv2 ($1.16$ to $0.95$).
Interaction variance is lowest for the MMD-based estimators (DINO+MMD, $V_{\text{m}}{=}0.62$), while \ours-CLIP improves on CMMD ($1.50$ vs.\ $1.88$).

\begin{table*}[ht!]
\centering
\caption{\textbf{Intra- and cross-dataset stability.}
We report \(\bar{\nu}_{k,k}\) (CV) and detected degradations percentage
(Det.), together with the cross-domain consistency \(\Lambda\)
and mean interaction variance \(V_{\text{m}}\).
Best in \textbf{bold}, second best \underline{underlined}.
}
\label{tab:stability-merged}
\scriptsize
\setlength{\tabcolsep}{1pt}
\renewcommand{\arraystretch}{1.1}

\resizebox{\textwidth}{!}{
\begin{tabular}{
@{}l
*{6}{r@{\hspace{5pt}}c}
>{\columncolor{yellow!15}}c
>{\columncolor{yellow!15}}c
>{\columncolor{yellow!15}}c
@{}
}
\toprule

& \multicolumn{13}{c}{
\textbf{Intra-dataset}: 
\(\bar{\nu}_{k,k}\) (CV) \,/\, Det. (\%)
}
& \multicolumn{2}{c}{\textbf{Cross-dataset}} \\

\cmidrule(lr){2-14}
\cmidrule(lr){15-16}

Metric
& \multicolumn{2}{c}{\textbf{AGIQA-3K}}
& \multicolumn{2}{c}{\textbf{CelebA-HQ}}
& \multicolumn{2}{c}{\textbf{ChestX-ray}}
& \multicolumn{2}{c}{\textbf{COCO-30k}}
& \multicolumn{2}{c}{\textbf{HAM10000}}
& \multicolumn{2}{c}{\textbf{RESISC45}}
& \cellcolor{yellow!15}\shortstack{\textbf{Avg Det.}\\\textbf{(\%)}}
& \cellcolor{yellow!15}\(\Lambda\downarrow\)
& \cellcolor{yellow!15}\(V_{\text{m}}\downarrow\) \\

\midrule

FID
& 61.52 (0.008) & 80.0
& 24.30 (0.016) & 86.7
& 17.41 (0.009) & 96.7
& 59.20 (0.011) & 73.3
& 24.51 (0.015) & 93.3
& 48.61 (0.010) & 93.3
& 87.2
& \underline{0.923}
& 1.14 \\

KID
& 1.54e-5 (6.64) & 86.7
& 2.70e-5 (5.58) & 86.7
& -2.79e-5 (2.93) & 96.7
& 3.18e-5 (3.42) & 83.3
& -6.24e-7 (375.22) & 90.0
& -1.45e-5 (7.93) & 93.3
& 89.4
& 0.969
& \underline{0.80} \\

CMMD
& 1.51 (0.095) & {100.0}
& 1.53 (0.096) & {93.3}
& 3.60e-1 (0.249) & 96.7
& 1.62 (0.090) & {96.7}
& 7.79e-1 (0.175) & {96.7}
& 1.19 (0.125) & {96.7}
& \textbf{96.7}
& 1.066
& 1.88 \\

DINO+MMD
& 1.30e-5 (4.32) & 90.0
& 5.47e-6 (20.91) & 86.7
& -3.94e-5 (4.46) & {100.0}
& 1.03e-5 (4.39) & {96.7}
& -5.65e-6 (22.32) & 93.3
& -3.81e-5 (1.85) & 93.3
& \underline{93.3}
& 1.162
& \textbf{0.62} \\

FWD
& 5.80 (0.015) & 43.3
& 2.40 (0.026) & 53.3
& 6.71e-1 (0.073) & 46.7
& 4.41 (0.018) & 53.3
& 5.51e-1 (0.046) & 56.7
& 2.63 (0.017) & 86.7
& 56.7
& 1.355
& 1.91 \\

LPIPS
& 7.39e-1 (0.003) & 60.0
& 5.42e-1 (0.004) & 56.7
& 4.26e-1 (0.004) & 80.0
& 7.44e-1 (0.002) & 56.7
& 4.82e-1 (0.007) & 73.3
& 6.48e-1 (0.012) & 73.3
& 66.7
& 1.782
& 3.97 \\

\rowcolor{gray!10}
\ours-CLIP
& 3.41e-2 (0.036) & {100.0}
& 3.35e-2 (0.038) & {93.3}
& 1.42e-2 (0.055) & 96.7
& 3.61e-2 (0.036) & {96.7}
& 2.00e-2 (0.044) & {96.7}
& 2.76e-2 (0.036) & {96.7}
& \cellcolor{oursummary}\textbf{96.7}
& \cellcolor{oursummary}\textbf{0.868}
& \cellcolor{oursummary}1.50 \\

\rowcolor{gray!10}
\ours-DINO
& 1.36e-3 (0.025) & 90.0
& 1.15e-3 (0.051) & 86.7
& 3.02e-4 (0.070) & {100.0}
& 1.38e-3 (0.024) & {96.7}
& 8.73e-4 (0.059) & 93.3
& 1.16e-3 (0.036) & 93.3
& \cellcolor{oursummary}\underline{93.3}
& \cellcolor{oursummary}0.945
& \cellcolor{oursummary}1.91 \\

\bottomrule
\end{tabular}
}
\end{table*}

\textbf{Sample Efficiency and Scalability.}
Figure~\ref{fig:complexity_grid} reports convergence and runtime versus sample size under Gaussian blur. 
\ours-DINO decreases monotonically with \(N\) on all six datasets (\ours-CLIP on five of six), with the same trend as CMMD.
FID and FWD, in contrast, carry a large positive finite-sample bias that decays slowly.
KID, DINO+MMD and LPIPS are unbiased in \(N\); their normalized curves mostly reflect subsampling noise, which the per-curve normalization magnifies.
The cost of the distance itself differs by orders of magnitude.
With \(L{=}512\) projections, \ours{} takes between 3 and 70\,ms at every \(N\) up to \(10^4\), while MMD-based metrics grow as $N^2$.
FID and FWD pay a fixed cost of 11--19\,s and about 5\,s, respectively, dominated by matrix square roots on fixed-size covariances, and KID saturates for $N{\geq}1000$.

\begin{figure*}[htpb!]
  \centering
  \figura[width=\textwidth]{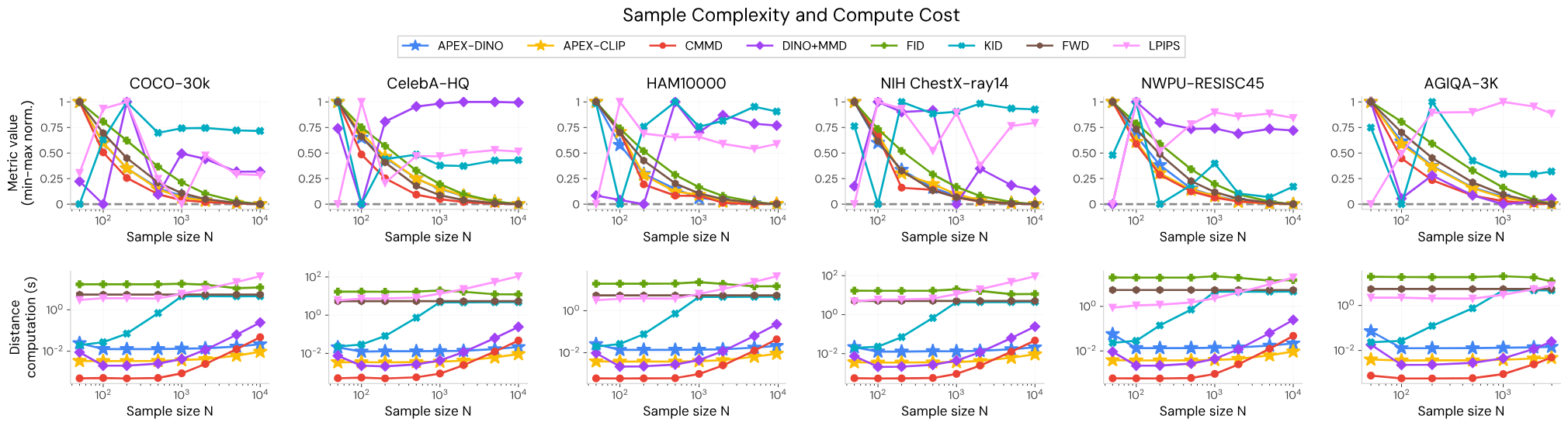}
  \caption{
    \textbf{Sample complexity and distance cost.}
    Real images vs.\ their Gaussian-blur (severity 2) version. 
    \(N\) images per side, 10 random subsamples per \(N\).
    \emph{Top:} metric value, min--max normalized per curve.
    \emph{Bottom:} time to compute the distance from pre-extracted features; for LPIPS, which has no separate feature extraction, the whole pipeline is timed.
    \ours{} uses \(L{=}512\) projections.}
    \label{fig:complexity_grid}
\end{figure*}

\paragraph{Consistency Across Coarse-to-Fine Generation.} 
\Cref{fig:progressive_generation_metrics} in Appendix~\ref{sec:app:img_refinem} shows the metrics' responses throughout the generative process of Stable Diffusion v1.5. 
As time steps increase up to $t{=}95$, all metrics show a clear downward trajectory, consistently reflecting an improvement in generative fidelity.
Focusing on the final stages of generation, \ours-CLIP draws the most significant drop, suggesting higher sensibility to subtle perceptual refinements, while \ours-DINO and CMMD roughly exhibit the same trend.

\paragraph{Human Perceptual Correlation Study.}
 We show that \ours is highly competitive on degradations, nearly matching FID and FWD performances and clearly improving over the others (\Cref{tab:metric-correlations-userstudy}). 
On generated images instead, modern feature-based metrics significantly establish over FID, FWD and KID.
On average, \ours metrics show highest correlation to human perception after LPIPS.
\vspace{-0.5em}
\begin{table}[htpb]
\centering
\caption{\textbf{Correlation with human perception.}
Spearman \(\rho\) and Kendall \(\tau\) (with \(p\)-values) between each metric and human judgment.
Best in \textbf{bold}, second \underline{underlined}.}
\label{tab:metric-correlations-userstudy}
\scriptsize
\setlength{\tabcolsep}{3pt}

\resizebox{\columnwidth}{!}{
\begin{tabular}{
l
cccc
cccc
>{\columncolor{summaryyellow}}c
>{\columncolor{summaryyellow}}c
}
\toprule

& \multicolumn{4}{c}{\textbf{Degradations}}
& \multicolumn{4}{c}{\textbf{AGIQA-3K}}
& \multicolumn{2}{>{\columncolor{summaryyellow}}c}{\textbf{Average}} \\

\cmidrule(lr){2-5}
\cmidrule(lr){6-9}
\cmidrule(lr){10-11}

\textbf{Metric}
& \(\rho\uparrow\)
& \(p_\rho\)
& \(\tau\uparrow\)
& \(p_\tau\)
& \(\rho\uparrow\)
& \(p_\rho\)
& \(\tau\uparrow\)
& \(p_\tau\)
& \cellcolor{summaryyellow}\(\boldsymbol{\rho}\uparrow\)
& \cellcolor{summaryyellow}\(\boldsymbol{\tau}\uparrow\) \\

\midrule

FID
& \(0.847\) & \(2.76e{-40}\) & \(0.669\) & \(1.03e{-30}\)
& \(0.371\) & \(1.07e{-1}\) & \(0.295\) & \(7.40e{-2}\)
& \(0.609\) & \(0.482\) \\

KID
& \(0.807\) & \(8.19e{-34}\) & \(0.622\) & \(8.08e{-27}\)
& \(0.126\) & \(5.96e{-1}\) & \(0.095\) & \(5.86e{-1}\)
& \(0.467\) & \(0.359\) \\

CMMD
& \(0.804\) & \(2.25e{-33}\) & \(0.626\) & \(4.37e{-27}\)
& \(0.626\) & \(3.18e{-3}\) & \(0.484\) & \(2.35e{-3}\)
& \(0.715\) & \(0.555\) \\

DINO+MMD
& \(0.649\) & \(2.51e{-18}\) & \(0.479\) & \(1.53e{-16}\)
& \(0.707\) & \(4.94e{-4}\) & \(0.537\) & \(6.38e{-4}\)
& \(0.678\) & \(0.508\) \\

FWD
& \(0.844\) & \(1.28e{-39}\) & \(0.665\) & \(1.98e{-30}\)
& \(0.245\) & \(2.98e{-1}\) & \(0.158\) & \(3.51e{-1}\)
& \(0.545\) & \(0.412\) \\

\rowcolor{gray!10}
\ours-CLIP
& \(0.830\) & \(2.52e{-37}\) & \(0.656\) & \(1.18e{-29}\)
& \(0.630\) & \(2.90e{-3}\) & \(0.495\) & \(1.84e{-3}\)
& \cellcolor{oursummary}\underline{0.730}
& \cellcolor{oursummary}\underline{0.576} \\

\rowcolor{gray!10}
\ours-DINO
& \(0.812\) & \(1.41e{-34}\) & \(0.630\) & \(1.94e{-27}\)
& \(0.705\) & \(5.14e{-4}\) & \(0.568\) & \(2.66e{-4}\)
& \cellcolor{oursummary}\(\mathbf{0.759}\)
& \cellcolor{oursummary}\(\mathbf{0.599}\) \\

\midrule

\multicolumn{9}{l}{\textit{Paired reference}}
& \cellcolor{summaryyellow}{}
& \cellcolor{summaryyellow}{} \\

LPIPS
& \(0.922\) & \(1.22e{-59}\) & \(0.762\) & \(2.48e{-39}\)
& \(0.747\)\(^\dagger\) & \(1.52e{-4}\) & \(0.611\) & \(7.32e{-5}\)
& \(0.835\) & \(0.687\) \\

\bottomrule
\end{tabular}
}

{\scriptsize
\(^\dagger\)AGIQA-3K has no paired reference: LPIPS compares each generated image with an unrelated real image.}
\end{table}

\paragraph{Limitations.}
\label{subsubsec:limitations}
Although APEX reduces reliance on
distributional assumptions and kernel hyperparameters, it depends on the quality of the model chosen for image features extraction.
Moreover, despite the efficiency of SWD, the overall cost of APEX is 
dominated by feature extraction.
This is particularly relevant for APEX-DINO, which leverages high-dimensional concatenated DINOv2 representations.
The human perceptual study suggests that APEX-CLIP is competitive with strong baselines, but larger and more diverse annotations pools could draw stronger conclusions on 
perceptual alignment.
\section{Conclusions} \label{sec:conclusion}
\vspace{-1em}
We introduced APEX, a framework for generative distribution evaluation
combining Sliced Wasserstein Distance with foundation model embeddings,
reducing reliance on parametric assumptions and kernel hyperparameters.
We provided theoretical insights and empirical evidence on the number
of projections required for stable and efficient SWD estimation.
Experiments across heterogeneous visual domains 
prove that APEX shows stronger
degradation sensitivity than established baselines 
while improving
cross-dataset consistency.
Moreover, APEX provides efficient distance computation and remains
sensitive to subtle refinements during image generation.
Human perceptual evaluations further confirm its competitiveness,
with APEX-DINO achieving the highest average correlation among
distributional metrics.
Overall, APEX offers a robust and efficient alternative for
generative evaluation across diverse visual domains.
Future work will explore the extension of APEX to broader multimodal models and adaptive projection optimization strategies, such as sampling primarly along directions of maximal divergence.
\bibliography{references}

\begin{thebibliography}{44}
\providecommand{\natexlab}[1]{#1}
\providecommand{\url}[1]{\texttt{#1}}
\expandafter\ifx\csname urlstyle\endcsname\relax
  \providecommand{\doi}[1]{doi: #1}\else
  \providecommand{\doi}{doi: \begingroup \urlstyle{rm}\Url}\fi

\bibitem[Ansuini et~al.(2019)Ansuini, Laio, Macke, and Zoccolan]{ansuini2019intrinsic}
Alessio Ansuini, Alessandro Laio, Jakob~H Macke, and Davide Zoccolan.
\newblock Intrinsic dimension of data representations in deep neural networks.
\newblock \emph{Advances in Neural Information Processing Systems}, 32, 2019.

\bibitem[Bi{\'n}kowski et~al.(2018)Bi{\'n}kowski, Sutherland, Arbel, and Gretton]{binkowski2018demystifying}
Miko{\l}aj Bi{\'n}kowski, Danica~J Sutherland, Michael Arbel, and Arthur Gretton.
\newblock Demystifying mmd gans.
\newblock In \emph{International Conference on Learning Representations}, 2018.

\bibitem[Bischoff et~al.(2024)Bischoff, Darcher, Deistler, Gao, Gerken, Gloeckler, Haxel, Kapoor, Lappalainen, Macke, et~al.]{bischoff2024practical}
Sebastian Bischoff, Alana Darcher, Michael Deistler, Richard Gao, Franziska Gerken, Manuel Gloeckler, Lisa Haxel, Jaivardhan Kapoor, Janne~K Lappalainen, Jakob~H Macke, et~al.
\newblock A practical guide to sample-based statistical distances for evaluating generative models in science.
\newblock \emph{Transactions on Machine Learning Research}, 2024.

\bibitem[Cheng et~al.(2017)Cheng, Han, and Lu]{cheng2017nwpu}
Gong Cheng, Junwei Han, and Xiaoqiang Lu.
\newblock Nwpu-resisc45: An open, large-scale dataset for scene classification in high-resolution remote sensing images.
\newblock \emph{Proceedings of the IEEE}, 105\penalty0 (1):\penalty0 145--166, 2017.

\bibitem[Deng et~al.(2009)Deng, Dong, Socher, Li, Li, and Fei-Fei]{deng2009imagenet}
Jia Deng, Wei Dong, Richard Socher, Li-Jia Li, Kai Li, and Li~Fei-Fei.
\newblock Imagenet: A large-scale hierarchical image database.
\newblock In \emph{2009 IEEE conference on computer vision and pattern recognition}, pages 248--255. Ieee, 2009.

\bibitem[Dhariwal and Nichol(2021)]{dhariwal2021diffusion}
Prafulla Dhariwal and Alexander Nichol.
\newblock Diffusion models beat {GANs} on image synthesis.
\newblock In \emph{Advances in Neural Information Processing Systems}, volume~34, pages 8780--8794, 2021.

\bibitem[Eftekhari and Wakin(2015)]{eftekhari2015new}
Armin Eftekhari and Michael~B Wakin.
\newblock New analysis of manifold embeddings and signal recovery from compressive measurements.
\newblock \emph{Applied and Computational Harmonic Analysis}, 39\penalty0 (1):\penalty0 67--109, 2015.

\bibitem[Goodfellow et~al.(2014)Goodfellow, Pouget-Abadie, Mirza, Xu, Warde-Farley, Ozair, Courville, and Bengio]{goodfellow2014generative}
Ian Goodfellow, Jean Pouget-Abadie, Mehdi Mirza, Bing Xu, David Warde-Farley, Sherjil Ozair, Aaron Courville, and Yoshua Bengio.
\newblock Generative adversarial nets.
\newblock In \emph{Advances in Neural Information Processing Systems}, volume~27, 2014.

\bibitem[Heusel et~al.(2017)Heusel, Ramsauer, Unterthiner, Nessler, and Hochreiter]{heusel2017gans}
Martin Heusel, Hubert Ramsauer, Thomas Unterthiner, Bernhard Nessler, and Sepp Hochreiter.
\newblock Gans trained by a two time-scale update rule converge to a local nash equilibrium.
\newblock \emph{Advances in neural information processing systems}, 30, 2017.

\bibitem[Ho et~al.(2020)Ho, Jain, and Abbeel]{ho2020denoising}
Jonathan Ho, Ajay Jain, and Pieter Abbeel.
\newblock Denoising diffusion probabilistic models.
\newblock In \emph{Advances in Neural Information Processing Systems}, volume~33, pages 6840--6851, 2020.

\bibitem[Jayasumana et~al.(2024)Jayasumana, Ramalingam, Veit, Glasner, Chakrabarti, and Kumar]{jayasumana2024rethinking}
Sadeep Jayasumana, Srikumar Ramalingam, Andreas Veit, Daniel Glasner, Ayan Chakrabarti, and Sanjiv Kumar.
\newblock Rethinking fid: Towards a better evaluation metric for image generation, 2024.

\bibitem[Karras et~al.(2018)Karras, Aila, Laine, and Lehtinen]{karras2018progressive}
Tero Karras, Timo Aila, Samuli Laine, and Jaakko Lehtinen.
\newblock Progressive growing of {GAN}s for improved quality, stability, and variation.
\newblock In \emph{International Conference on Learning Representations}, 2018.
\newblock URL \url{https://openreview.net/forum?id=Hk99zCeAb}.

\bibitem[Karras et~al.(2019)Karras, Laine, and Aila]{karras2019style}
Tero Karras, Samuli Laine, and Timo Aila.
\newblock A style-based generator architecture for generative adversarial networks.
\newblock In \emph{Proceedings of the IEEE/CVF Conference on Computer Vision and Pattern Recognition}, pages 4401--4410, 2019.

\bibitem[Kendall(1938)]{kendall_tau}
M.~G. Kendall.
\newblock A new measure of rank correlation.
\newblock \emph{Biometrika}, 30\penalty0 (1-2):\penalty0 81--93, 06 1938.
\newblock ISSN 0006-3444.
\newblock \doi{10.1093/biomet/30.1-2.81}.
\newblock URL \url{https://doi.org/10.1093/biomet/30.1-2.81}.

\bibitem[Kwon et~al.(2026)Kwon, Yang, and Chae]{kwon2026evaluating}
Hyeok~Kyu Kwon, Jaeseung Yang, and Minwoo Chae.
\newblock Evaluating image generation models via sliced wasserstein distance.
\newblock \emph{Journal of the Korean Statistical Society}, pages 1--21, 2026.

\bibitem[Lee et~al.(2019)Lee, Batra, Baig, and Ulbricht]{lee2019sliced}
Chen-Yu Lee, Tanmay Batra, Mohammad~Haris Baig, and Daniel Ulbricht.
\newblock Sliced wasserstein discrepancy for unsupervised domain adaptation.
\newblock In \emph{2019 IEEE/CVF conference on computer vision and pattern recognition (CVPR)}, pages 10277--10287. IEEE, 2019.

\bibitem[Li et~al.(2023)Li, Zhang, Wu, Sun, Min, Liu, Zhai, and Lin]{li2023agiqa}
Chunyi Li, Zicheng Zhang, Haoning Wu, Wei Sun, Xiongkuo Min, Xiaohong Liu, Guangtao Zhai, and Weisi Lin.
\newblock Agiqa-3k: An open database for ai-generated image quality assessment.
\newblock \emph{IEEE Transactions on Circuits and Systems for Video Technology}, 34\penalty0 (8):\penalty0 6833--6846, 2023.

\bibitem[Lin et~al.(2014)Lin, Maire, Belongie, Hays, Perona, Ramanan, Doll{\'a}r, and Zitnick]{lin2014microsoft}
Tsung-Yi Lin, Michael Maire, Serge Belongie, James Hays, Pietro Perona, Deva Ramanan, Piotr Doll{\'a}r, and C~Lawrence Zitnick.
\newblock Microsoft coco: Common objects in context.
\newblock In \emph{European conference on computer vision}, pages 740--755. Springer, 2014.

\bibitem[Liu et~al.(2015)Liu, Luo, Wang, and Tang]{liu2015deep}
Ziwei Liu, Ping Luo, Xiaogang Wang, and Xiaoou Tang.
\newblock Deep learning face attributes in the wild.
\newblock In \emph{Proceedings of the IEEE international conference on computer vision}, pages 3730--3738, 2015.

\bibitem[Oquab et~al.(2024)Oquab, Darcet, Moutakanni, Vo, Szafraniec, Khalidov, Fernandez, HAZIZA, Massa, El-Nouby, Assran, Ballas, Galuba, Howes, Huang, Li, Misra, Rabbat, Sharma, Synnaeve, Xu, Jegou, Mairal, Labatut, Joulin, and Bojanowski]{oquab2024dinov}
Maxime Oquab, Timoth{\'e}e Darcet, Th{\'e}o Moutakanni, Huy~V. Vo, Marc Szafraniec, Vasil Khalidov, Pierre Fernandez, Daniel HAZIZA, Francisco Massa, Alaaeldin El-Nouby, Mido Assran, Nicolas Ballas, Wojciech Galuba, Russell Howes, Po-Yao Huang, Shang-Wen Li, Ishan Misra, Michael Rabbat, Vasu Sharma, Gabriel Synnaeve, Hu~Xu, Herve Jegou, Julien Mairal, Patrick Labatut, Armand Joulin, and Piotr Bojanowski.
\newblock {DINO}v2: Learning robust visual features without supervision.
\newblock \emph{Transactions on Machine Learning Research}, 2024.
\newblock ISSN 2835-8856.
\newblock URL \url{https://openreview.net/forum?id=a68SUt6zFt}.
\newblock Featured Certification.

\bibitem[Peebles and Xie(2023)]{peebles2023scalable}
William Peebles and Saining Xie.
\newblock Scalable diffusion models with transformers.
\newblock In \emph{Proceedings of the IEEE/CVF International Conference on Computer Vision}, pages 4199--4209, 2023.

\bibitem[Peyr{\'e} et~al.(2019)Peyr{\'e}, Cuturi, et~al.]{peyre2019computational}
Gabriel Peyr{\'e}, Marco Cuturi, et~al.
\newblock Computational optimal transport: With applications to data science.
\newblock \emph{Foundations and Trends{\textregistered} in Machine Learning}, 11\penalty0 (5-6):\penalty0 355--607, 2019.

\bibitem[Pope et~al.(2021)Pope, Zhu, Abdelkader, Goldblum, and Goldstein]{pope2021intrinsic}
Phil Pope, Chen Zhu, Ahmed Abdelkader, Micah Goldblum, and Tom Goldstein.
\newblock The intrinsic dimension of images and its impact on learning.
\newblock In \emph{International Conference on Learning Representations}, 2021.
\newblock URL \url{https://openreview.net/forum?id=XJk19XzGq2J}.

\bibitem[Rabin et~al.(2011)Rabin, Peyr{\'e}, Delon, and Bernot]{rabin2011wasserstein}
Julien Rabin, Gabriel Peyr{\'e}, Julie Delon, and Marc Bernot.
\newblock Wasserstein barycenter and its application to texture mixing.
\newblock In \emph{International conference on scale space and variational methods in computer vision}, pages 435--446. Springer, 2011.

\bibitem[Radford et~al.(2021)Radford, Kim, Hallacy, Ramesh, Goh, Agarwal, Sastry, Askell, Mishkin, Clark, et~al.]{radford2021learning}
Alec Radford, Jong~Wook Kim, Chris Hallacy, Aditya Ramesh, Gabriel Goh, Sandhini Agarwal, Girish Sastry, Amanda Askell, Pamela Mishkin, Jack Clark, et~al.
\newblock Learning transferable visual models from natural language supervision.
\newblock In \emph{International conference on machine learning}, pages 8748--8763. PmLR, 2021.

\bibitem[Ramesh et~al.(2022)Ramesh, Dhariwal, Nichol, Chu, and Chen]{ramesh2022hierarchical}
Aditya Ramesh, Prafulla Dhariwal, Alex Nichol, Casey Chu, and Mark Chen.
\newblock Hierarchical text-conditional image generation with {CLIP} latents.
\newblock \emph{arXiv preprint arXiv:2204.06125}, 2022.

\bibitem[Rombach et~al.(2022)Rombach, Blattmann, Lorenz, Esser, and Ommer]{rombach2022high}
Robin Rombach, Andreas Blattmann, Dominik Lorenz, Patrick Esser, and Bj{\"o}rn Ommer.
\newblock High-resolution image synthesis with latent diffusion models.
\newblock In \emph{Proceedings of the IEEE/CVF conference on computer vision and pattern recognition}, pages 10684--10695, 2022.

\bibitem[Salimans et~al.(2016)Salimans, Goodfellow, Zaremba, Cheung, Radford, and Chen]{salimans2016improved}
Tim Salimans, Ian Goodfellow, Wojciech Zaremba, Vicki Cheung, Alec Radford, and Xi~Chen.
\newblock Improved techniques for training gans.
\newblock \emph{Advances in neural information processing systems}, 29, 2016.

\bibitem[Sohl-Dickstein et~al.(2015)Sohl-Dickstein, Weiss, Maheswaranathan, and Ganguli]{sohl2015deep}
Jascha Sohl-Dickstein, Eric Weiss, Niru Maheswaranathan, and Surya Ganguli.
\newblock Deep unsupervised learning using nonequilibrium thermodynamics.
\newblock In \emph{International Conference on Machine Learning}, pages 2256--2265. PMLR, 2015.

\bibitem[Song and Ermon(2019)]{song2019generative}
Yang Song and Stefano Ermon.
\newblock Generative modeling by estimating gradients of the data distribution.
\newblock In \emph{Advances in Neural Information Processing Systems}, volume~32, 2019.

\bibitem[Song et~al.(2021)Song, Sohl-Dickstein, Kingma, Kumar, Ermon, and Poole]{song2021score}
Yang Song, Jascha Sohl-Dickstein, Diederik~P Kingma, Arvind Kumar, Stefano Ermon, and Ben Poole.
\newblock Score-based generative modeling through stochastic differential equations.
\newblock In \emph{International Conference on Learning Representations}, 2021.

\bibitem[Spearman(1904)]{spearman_ro}
C.~Spearman.
\newblock The proof and measurement of association between two things.
\newblock \emph{The American Journal of Psychology}, 15\penalty0 (1):\penalty0 72--101, 1904.
\newblock ISSN 00029556.
\newblock URL \url{http://www.jstor.org/stable/1412159}.

\bibitem[Stein et~al.(2023)Stein, Cresswell, Hosseinzadeh, Sui, Ross, Villecroze, Liu, Caterini, Taylor, and Loaiza-Ganem]{stein2023exposing}
George Stein, Jesse Cresswell, Rasa Hosseinzadeh, Yi~Sui, Brendan Ross, Valentin Villecroze, Zhaoyan Liu, Anthony~L Caterini, Eric Taylor, and Gabriel Loaiza-Ganem.
\newblock Exposing flaws of generative model evaluation metrics and their unfair treatment of diffusion models.
\newblock \emph{Advances in Neural Information Processing Systems}, 36:\penalty0 3732--3784, 2023.

\bibitem[Streijl et~al.(2016)Streijl, Winkler, and Hands]{streijl2016mean}
Robert~C Streijl, Stefan Winkler, and David~S Hands.
\newblock Mean opinion score (mos) revisited: methods and applications, limitations and alternatives.
\newblock \emph{Multimedia Systems}, 22\penalty0 (2):\penalty0 213--227, 2016.

\bibitem[Szegedy et~al.(2016)Szegedy, Vanhoucke, Ioffe, Shlens, and Wojna]{szegedy2016rethinking}
Christian Szegedy, Vincent Vanhoucke, Sergey Ioffe, Jon Shlens, and Zbigniew Wojna.
\newblock Rethinking the inception architecture for computer vision.
\newblock In \emph{Proceedings of the IEEE conference on computer vision and pattern recognition}, pages 2818--2826, 2016.

\bibitem[Tanchenko(2014)]{tanchenko2014visual}
Alexander Tanchenko.
\newblock Visual-psnr measure of image quality.
\newblock \emph{Journal of Visual Communication and Image Representation}, 25\penalty0 (5):\penalty0 874--878, 2014.

\bibitem[Theis et~al.(2016)Theis, van~den Oord, and Bethge]{TheisOB15}
L.~Theis, A.~van~den Oord, and M.~Bethge.
\newblock A note on the evaluation of generative models.
\newblock In \emph{International Conference on Learning Representations}, Apr 2016.
\newblock URL \url{http://arxiv.org/abs/1511.01844}.

\bibitem[Tschandl et~al.(2018)Tschandl, Rosendahl, and Kittler]{tschandl2018ham10000}
Philipp Tschandl, Cliff Rosendahl, and Harald Kittler.
\newblock The ham10000 dataset, a large collection of multi-source dermatoscopic images of common pigmented skin lesions.
\newblock \emph{Scientific data}, 5\penalty0 (1):\penalty0 180161, 2018.

\bibitem[Veeramacheneni et~al.(2025)Veeramacheneni, Wolter, Kuehne, and Gall]{veeramacheneni2025frechet}
Lokesh Veeramacheneni, Moritz Wolter, Hilde Kuehne, and Juergen Gall.
\newblock Fr{\'e}chet wavelet distance: A domain-agnostic metric for image generation.
\newblock In \emph{International Conference on Learning Representations}, volume 2025, pages 41235--41265, 2025.

\bibitem[Wang et~al.(2017)Wang, Peng, Lu, Lu, Bagheri, and Summers]{chestxray2017wang}
Xiaosong Wang, Yifan Peng, Le~Lu, Zhiyong Lu, Mohammadhadi Bagheri, and Ronald Summers.
\newblock Chestx-ray14: Hospital-scale chest x-ray database and benchmarks on weakly-supervised classification and localization of common thorax diseases.
\newblock 09 2017.

\bibitem[Wang and Bovik(2009)]{wang2009mean}
Zhou Wang and Alan~C Bovik.
\newblock Mean squared error: Love it or leave it? a new look at signal fidelity measures.
\newblock \emph{IEEE signal processing magazine}, 26\penalty0 (1):\penalty0 98--117, 2009.

\bibitem[Wang et~al.(2004)Wang, Bovik, Sheikh, and Simoncelli]{wang2004image}
Zhou Wang, Alan~C Bovik, Hamid~R Sheikh, and Eero~P Simoncelli.
\newblock Image quality assessment: from error visibility to structural similarity.
\newblock \emph{IEEE transactions on image processing}, 13\penalty0 (4):\penalty0 600--612, 2004.

\bibitem[Yang et~al.(2023)Yang, Zhang, Song, Hong, Xu, Zhao, Zhang, Cui, and Yang]{yang2023diffusion}
Ling Yang, Zhilong Zhang, Yang Song, Shenda Hong, Runsheng Xu, Yue Zhao, Wentao Zhang, Bin Cui, and Ming-Hsuan Yang.
\newblock Diffusion models: A comprehensive survey of methods and applications.
\newblock \emph{ACM Computing Surveys}, 56\penalty0 (4):\penalty0 1--39, 2023.

\bibitem[Zhang et~al.(2018)Zhang, Isola, Efros, Shechtman, and Wang]{zhang2018unreasonable}
Richard Zhang, Phillip Isola, Alexei~A Efros, Eli Shechtman, and Oliver Wang.
\newblock The unreasonable effectiveness of deep features as a perceptual metric.
\newblock In \emph{Proceedings of the IEEE conference on computer vision and pattern recognition}, pages 586--595, 2018.

\end{thebibliography}

\appendix

\section{Proof of Theorem 1}
\label{app:new_thm}
This section provides the proof of \Cref{theorem}, reported again here for simplicity.

\setcounter{theorem}{0}
\begin{theorem} 
    Let $\mu$ and $\nu$ be data distributions supported on a bounded $k$-dimensional manifold $\mathcal {M} \subseteq\mathbb{R}^d$ with diameter $D>0$. 
    Assume there is a constant $A\geq 1$ such that, for every $0 < \epsilon \leq D$, the minimum cardinality of an
$\epsilon$-net with centers in $\mathcal M$ is $\mathcal N(\mathcal M, \epsilon)\leq \left(\frac{AD}{\epsilon}\right)^k$, using the ambient Euclidean distance.
    Let $SW_2(\mu_n,\nu_n)^2$ be the squared Sliced Wasserstein distance computed on the empirical distributions $\mu_n, \nu_n$ induced by $n$ samples, and $\widehat{SW_2}(\mu_n,\nu_n)^2$ its Monte Carlo estimate using $L$ uniformly sampled independent projections.
    Then, for any $\xi>0$ and $0<\delta<1$, if 
    \begin{equation}
        L \geq \frac{2D^4}{\xi^2} \left[ 2nk \log \left( \frac{16AD^2}{\xi} \right) + \log\left( \frac{2}{\delta} \right) \right]
    \end{equation}
    the error $|\widehat{SW_2}(\mu_n,\nu_n)^2-SW_2(\mu_n,\nu_n)^2|\leq \xi$ with probability at least $1-\delta$.
\end{theorem}

In what follows, we discuss the consistency of the assumptions in the experimental setting of APEX.
\begin{remark}[Image Embeddings] \label{rem_imageembeddings}
    Let $\mathcal{I}$ be the space of images and $P,Q \in \mathcal{P}(\mathcal{I})$ the corresponding data distributions. 
    Let $\phi: \mathcal{I} \rightarrow \mathbb{R}^d$ be a pre-trained embedding function yielding the latent distributions $\mu= \phi_\sharp P$ and $\nu= \phi_\sharp Q$. When their supports are contained in $\mathcal {M}$, a bounded $k$-dimensional manifold in $\mathbb{R}^d$, any collection of $n$ embeddings gives point clouds $X,Y \in \mathcal{M}^n$ to which the theorem applies. 
    In APEX, these embeddings are obtained using CLIP or DINOv2.
\end{remark}

\begin{remark}[Intrinsic Dimension] \label{rem:low_dim}
Let $d$ denote the ambient dimension of the latent space, i.e., the length of each feature vector, wheres $k$ is the intrinsic dimension of the supporting manifold $\mathcal{M}$.
When dealing with high-dimensional samples, like images, the  
manifold hypothesis states that data is actually embedded on, or near to, lower dimensional spaces embedded in the ambient. Empirical studies provide evidences to motivate this assumption \cite{ansuini2019intrinsic, pope2021intrinsic}.
Thus, one can generally assume $k < d$ for deep embeddings. 
In the trivial case, when $k=d$, the dimension of the relevant subspace supporting the embedding distribution equals the one of the ambient space.
\end{remark}

\begin{remark}[Boundedness Assumption] \label{assumption: boundedness}
Consider $x,y$ two images and $\phi(x),\phi(y) \in \mathcal{M}$ the corresponding embeddings computed via CLIP or DINOv2.
By design, CLIP has a final $\ell_2$-normalization layer and we normalize each layer of DINOv2 prior to concatenation. 
Therefore 
\begin{equation}
\|\phi(x)\|_2 \leq R \quad \text{and} \quad \|\phi(y)\|_2 \leq R
\end{equation}
resulting in $\| \phi(x)-\phi(y)\|_2 \leq 2R$, with $R$ finite. Thus,  $D=2R$ is a valid diameter for our manifold $\mathcal{M}$. 
In general, its exact value need not be known or estimated.
\end{remark}

\begin{remark}[Covering Complexity] \label{rem_constantA}
An $\epsilon$-net is a finite collection of representative embeddings such that
every point of $\mathcal M$ is within distance $\epsilon$ of a representative.
The bound in \Cref{theorem} on the covering number quantifies how many
representatives suffice at each resolution $\epsilon$. In particular, if $k$ is the intrinsic dimension of the manifold $\mathcal{M}$, the covering number can be assumed to scale with $O(1/\epsilon)^k$ (see, e.g.,~\cite{eftekhari2015new}).
While $k$ is the exponent describing dependence on resolution, $A$ accounts for the number of coordinate patches, their size, and the distortion of the parametrizations. 
The theorem does not require these factors to be identified, but the
covering assumption requires their geometric effect to be controlled.
\end{remark}

We can now discuss the proof of \Cref{theorem}.
\begin{proof}
The argument is made of three main steps: (i) a concentration inequality
controls the error for each fixed pair of point clouds; (ii) a
deterministic stability estimate shows that nearby clouds have similar
errors; (iii) a finite covering and a union bound make the concentration estimate simultaneous over all clouds.
\paragraph{Step (i)} First, we want to prove that a concentration inequality controls the error for each pair of fixed point clouds.
Let $\mu_n$ and $\nu_n$ be two empirical distributions induced by $n$ samples. We estimate the squared Sliced Wasserstein distance 
\begin{equation} 
    SW_2(\mu_n,\nu_n)^2= \int_{\mathbb{S}^{d-1}} \mathcal{W}_2(\pi_{\theta \sharp}\mu_n,\pi_{\theta \sharp}\nu_n)^2 d\theta 
\end{equation}
by using a simple Monte Carlo scheme that uniformly draws $L$ sample directions ${\theta_l}$ on $\mathbb{S}^{d-1}$, i.e.,
\begin{equation} \label{eq:sliced_wass_sum}
\widehat{SW_2}(\mu_n,\nu_n)^2=\frac{1}{L} \sum_{l=1}^L \mathcal{W}_2(\pi_{\theta_l \sharp}\mu_n,\pi_{\theta_l \sharp}\nu_n)^2 
\end{equation}
where $\pi_{\theta_l}: \mathbb{R}^d \rightarrow \mathbb{R}$ is the projection onto a vector $\theta_l \in \mathbb{S}^{d-1}$ of the $d$-dimensional unit sphere and $\pi_\theta(x)=\langle\theta,x\rangle$.

Now, consider $X= (x_1, \dots,x_n)$ and $Y=(y_1, \dots, y_n)$ two sets of $n$ samples inducing the empirical distributions
\begin{equation}
    \mu_{n,X}= \frac{1}{n} \sum_{i=1}^n \delta_{x_i} \quad \text{and} \quad \nu_{n,Y}= \frac{1}{n} \sum_{i=1}^n \delta_{y_i},
\end{equation}
where the subscripts refer to the specific sampling and its size.
Since each point $x_i,y_i \in \mathcal{M}$, then $X,Y \in \mathcal
M^n$ as the collection of $n$ samples and supp$(\mu_{n,X}) \subseteq \mathcal{M}$, supp$(\nu_{n,Y}) \subseteq \mathcal{M}$.
Let us denote with $f(\theta)$ the squared Wasserstein distance between distributions projected onto the direction $\theta$.
Then, from \citep{peyre2019computational}, the one-dimensional transport problem between equally weighted empirical measures can be written as
\begin{equation}
    f_{X,Y}(\theta)=\mathcal{W}_2(\pi_{\theta \sharp}\mu_{n,X},\pi_{\theta \sharp}\nu_{n,Y})^2= \min_{ \sigma \in S_n} \frac{1}{n} \sum_{i=1}^n \langle \theta, x_i-y_{\sigma(i)}\rangle^2
\end{equation}
where $S_n$ is the set of permutations of $n$ indices and the subscript $X,Y$ refers to the fixed point clouds.
The minimizing matching can be obtained by sorting the projected samples.
Since $\| \theta \|_2=1$ and $\|x_i-y_j\| \leq D$ from the boundedness assumption, every squared projected distance is in $[0,D^2]$ and therefore $0 \leq f_{X,Y}(\theta) \leq D^2.$
Denote the centered Monte Carlo estimation error as
\begin{equation}
    Z_L(X,Y) = \widehat{SW}_2(\mu_{n,X}, \nu_{n,Y})^2 - SW_2(\mu_{n,X}, \nu_{n,Y})^2.
\end{equation}
Since the random variables $f_{X,Y}(\theta_1), \dots, f_{X,Y}(\theta_L)$ are independent, take values in $[0,D^2]$, and have common mean $SW_2(\mu_{n,X}, \nu_{n,Y})^2$, we can apply the Hoeffding's inequality yielding
\begin{equation}
    \mathbb{P} \left( \left| \widehat{SW}(\mu_{n,X}, \nu_{n,Y})- SW(\mu_{n,X}, \nu_{n,Y}) \right| >t \right)= \mathbb{P} \left( |Z_L(X,Y)| >t\right) \leq 2e^{-2Lt^2/D^4}
\end{equation}
for every $t >0$.
\paragraph{Step (ii)} Then, we analyse the behaviour of the estimate computed on nearby point clouds.
Suppose now to have $X,Y,X',Y'\in\mathcal M^n$ point clouds of $n$ samples that satisfy
\begin{equation}
    \max_i \|x_i-x_i'\| \leq \epsilon \quad \text{and} \quad \max_i \|y_i-y_i'\| \leq \epsilon.
\end{equation}
Fix a direction $\theta$ and a permutation $\sigma$, we define $v_i:=x_i- y_{\sigma(i)}$ and $v_i':=x'_i- y'_{\sigma(i)}$. 
From boundedness of $\mathcal{M}$ we get $\|v_i\|, \| v_i'\| \leq D$, and
\begin{equation}
\|v_i-v_i'\| = \|(x_i- y_{\sigma(i)}) -(x'_i- y'_{\sigma(i)})\|\leq \|(x_i- y_{\sigma(i)})\|+\|(x'_i- y'_{\sigma(i)}) \| \leq2\epsilon.
\end{equation}
Factoring a difference of squares and applying Cauchy-Schwartz inequality gives
\begin{align} \label{eq:gen_diff}
   |\langle \theta, x_i-y_{\sigma(i)}\rangle^2- \langle \theta, x'_i-y'_{\sigma(i)}\rangle^2|  &= |\langle \theta, v_i\rangle^2- \langle \theta, v'_i\rangle^2| \nonumber \\ &= |\langle \theta, v_i-v_i'\rangle \langle \theta, v_i+v_i'\rangle| \nonumber \\ &\leq \|v_i-v_i'\| \|v_i+v_i'\| \leq 4D\epsilon.
\end{align}
Consider the averaged cost of permutation $\sigma$, then
\begin{equation}
    \left| \frac{1}{n} \sum_{i=1}^n \langle \theta, v_i\rangle^2- \frac{1}{n} \sum_{i=1}^n \langle \theta, v'_i\rangle^2 \right| \leq \frac{1}{n} \sum_{i=1}^n 4D\epsilon = 4D\epsilon
\end{equation}
which actually holds for every permutation. Stepping into the optimal costs, we can consider the following inequality (where we use $a_\sigma,b_\sigma$ to denote the cost of permutation $\sigma$, for the sake of simplicity):
\[
\left|\min_\sigma a_\sigma-\min_\sigma b_\sigma\right|
\leq\max_\sigma|a_\sigma-b_\sigma|.
\]
If $|a_\sigma-b_\sigma|\leq c$ for all $\sigma$, then
$a_\sigma\leq b_\sigma+c$ for all $\sigma$, so
$\min_\sigma a_\sigma\leq\min_\sigma b_\sigma+c$; reversing the roles gives exactly the opposite inequality. The optimal permutations need not coincide.
Thus, we get
\begin{equation}
   |f_{X,Y}(\theta)-f_{X',Y'}(\theta)| \leq 4D \epsilon.
\end{equation}
Similarly, we derive that 
\begin{equation}
    |\widehat{SW}(\mu_{n,X}, \nu_{n,Y})-\widehat{SW}(\mu_{n,X'}, \nu_{n,Y'}) | \leq \frac{1}{L} \sum_{l=1}^L |f_{X,Y}(\theta_l)-f_{X',Y'}(\theta_l)| \leq 4D \epsilon
\end{equation}
and
\begin{equation}
|SW(\mu_{n,X}, \nu_{n,Y})-SW(\mu_{n,X'}, \nu_{n,Y'}) |= |\mathbb{E}_\theta [f_{X,Y}(\theta)]- \mathbb{E}_\theta [f_{X',Y'}(\theta)]|\leq 4D \epsilon.
\end{equation}
Applying the triangle inequality to the difference of centered errors computed over different point clouds, we have
\begin{align}
    |Z_L(X,Y)-Z_L(X',Y')| &= \left|\left[\widehat{SW}(\mu_{n,X}, \nu_{n,Y})- SW(\mu_{n,X}, \nu_{n,Y})\right] \right. \nonumber \\ 
    & \left. \quad - \left[\widehat{SW}(\mu_{n,X'}, \nu_{n,Y'})- SW(\mu_{n,X'}, \nu_{n,Y'})\right]\right| \nonumber\\
    & \leq \left|\left[\widehat{SW}(\mu_{n,X}, \nu_{n,Y})-\widehat{SW}(\mu_{n,X'}, \nu_{n,Y'}) \right]  \right| \nonumber\\
    & \quad + \Big| \Big[SW(\mu_{n,X}, \nu_{n,Y})-SW(\mu_{n,X'}, \nu_{n,Y'}) \Big]  \Big| \nonumber \\
    & \leq 4D \epsilon+ 4D\epsilon = 8D\epsilon.
\end{align}
This allows us to extend a bound on a finite set of point clouds to every cloud.
\paragraph{Step (iii)} Finally, we consider a finite covering to characterise the estimation error over all point clouds.
Fix $0< \epsilon \leq D$. Then, by assumption, we can choose an $\epsilon$-net $\mathcal{A}_\epsilon \subseteq \mathcal{M}$ with the cardinality $|\mathcal{A}_\epsilon|\leq \left(A D/\epsilon \right)^k$.
Then, the $2n$ points of a pair $(X,Y)$ need a net of cardinality at most $|\mathcal{A}_\epsilon|^{2n}\leq \left(AD /\epsilon \right)^{2nk}$, as each point is replaced by a representative of $\mathcal{A}_\epsilon$.
For every pair $(X,Y)$, this replacement produces a pair $(X',Y') \in \mathcal{A}_\epsilon^{2n}$ satisfying the assumptions above.
Then, for any realization of the random directions, we have
\begin{equation}
    |Z_L(X,Y)| \leq |Z_L(X',Y')| +8D\epsilon \leq \max_{(U,V) \in \mathcal{A}_\epsilon^{2n}} |Z_L(U,V)| + 8D\epsilon
\end{equation}
then
\begin{equation}
    \sup_{X,Y \in \mathcal{M}^n}|Z_L(X,Y)| \leq \max_{(U,V) \in \mathcal{A}_\epsilon^{2n}} |Z_L(U,V)| + 8D\epsilon.
\end{equation}
In particular, if the left-hand side of the above exceeds $t+8D\epsilon$, for some $t>0$, at least one net pair must have error greater than $t$. This is the event inclusion
\[
\left\{\sup_{X,Y\in\mathcal M^n}|Z_L(X,Y)|>t+8D\epsilon\right\}
\subseteq
\bigcup_{(U,V)\in \mathcal{A}^{2n}_\epsilon}\{|Z_L(U,V)|>t\}.
\]
The union bound and the Hoeffding's inequality now give
\begin{align}
\mathbb P\left(\sup_{X,Y\in\mathcal M^n}|Z_L(X,Y)|>t+8D\epsilon\right)
&\leq\sum_{(U,V)\in\mathcal{A}_\epsilon^{2n}}
    \mathbb P\bigl(|Z_L(U,V)|>t\bigr)\leq 2 \left( \frac{AD}{\epsilon} \right)^{2nk} e^{-2Lt^2/D^4}.
\end{align} 
Since we want the total approximation error to be at most $\xi$, some positive tolerance value, choose
$\epsilon= \xi/16D$ and $t= \xi/2$ so that
\begin{equation}
    \mathbb{P} \left( \sup_{X,Y}|Z_L(X,Y)| >\xi\right) \leq 2 \left( \frac{16AD^2}{\xi} \right)^{2nk} e^{-L\xi^2/2D^4.}
\end{equation} 
In other words, we express the resolution of the net and the bound on the estimation-error in terms of $\xi$.
By setting the failure probability to $\delta$, we derive the following condition on $L$:
\begin{align} \label{eq:cond}
 &2 \left( \frac{16AD^2}{\xi} \right)^{2nk} e^{-L\xi^2/2D^4} \leq \delta \quad \Rightarrow \quad 2nk \log \left( \frac{16AD^2}{\xi} \right)- \frac{L \xi^2}{2D^4} \leq \log \left(\frac{\delta}{2} \right) \nonumber \\
 & \frac{L \xi^2}{2D^4} \geq 2nk \log \left( \frac{16AD^2}{\xi} \right) - \log \left(\frac{\delta}{2} \right) \quad \Rightarrow \quad L \geq \frac{2D^4}{\xi^2} \left[ 2nk \log \left( \frac{16AD^2}{\xi} \right) + \log \left( \frac{2}{\delta}\right) \right].
\end{align}
To conclude, the approximation error of the Monte Carlo estimation from $L$ projections, over empirical distribution of $n$ samples, is within accuracy $\xi$ with probability $1-\delta$ when \Cref{eq:cond} is satisfied.
\end{proof}

\section{Ablation Study on the number of projections for a stable SWD estimation} \label{ablation_nproj}
In this section, we present empirical evidences on the 
number of projections needed to estimate Sliced Wasserstein Distance via Monte Carlo in a stable and robust way (\Cref{app:fig:swd_projections}), for each dataset. 
Therefore, we investigate how $L$ influences the SWD computation at fixed different data structures.
We compute the SWD estimate with $L=\{10, 50, 70, 100, 500, 1000, 2000, 5000, 10000, 15000,$ $ 20000\}$ 
across datasets and image corruptions.
We show the trade-off between runtime and \ours metrics behaviour as $L$ increases. 
\begin{figure}[htpb]
    \centering
    \figura[width=\textwidth]{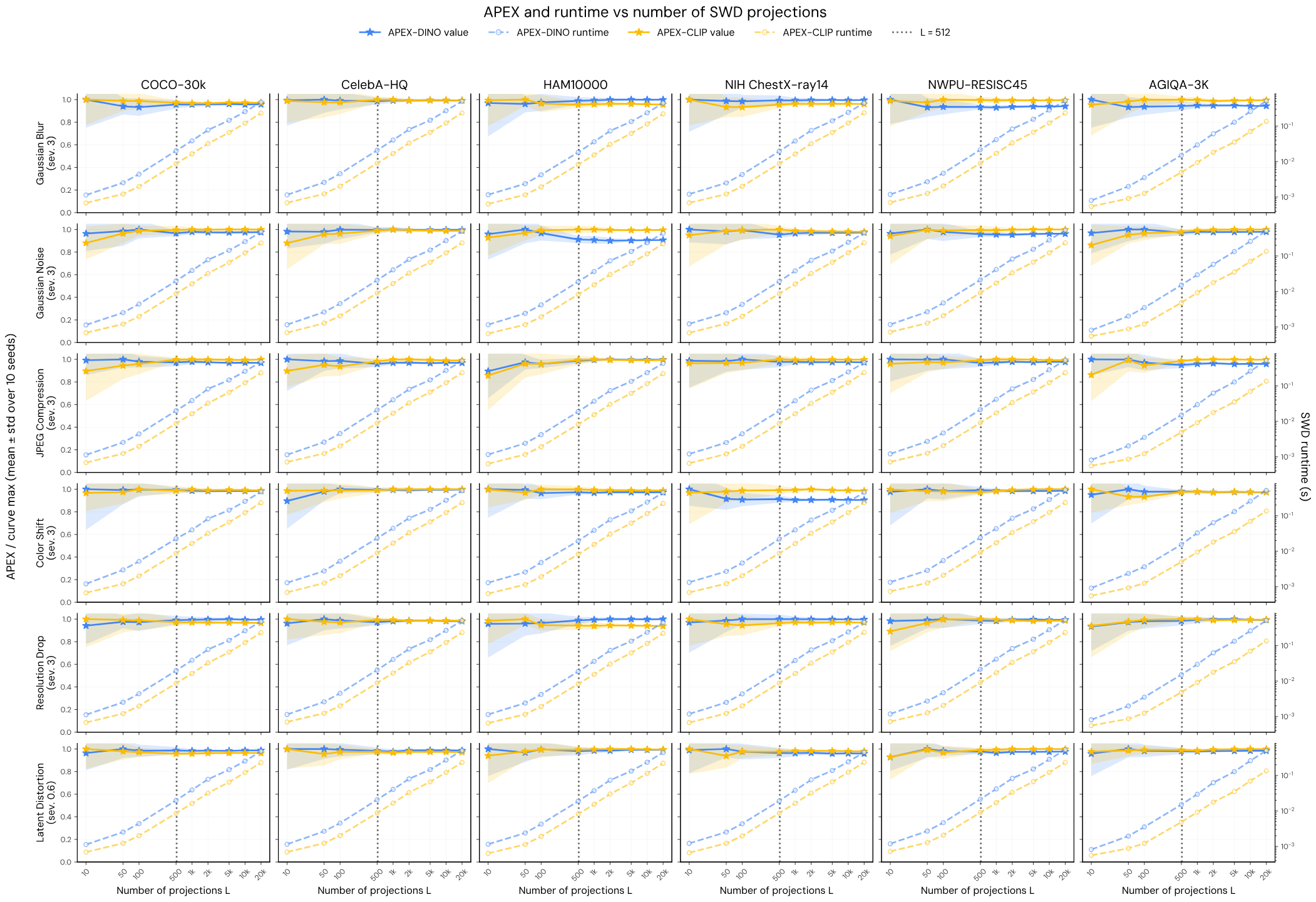}
    \caption{
        \textbf{Number of SWD projections vs. execution time for APEX across image corruptions and datasets.}
        We report the APEX metrics score (max-normalised) on the left axis, while dashed curves report execution time on the right axis. Both trends are shown as the number of projection $L$ used in the SWD computation increases.}
    \label{app:fig:swd_projections}
\end{figure}

\section{Degradation Sensitivity} \label{app:fig:degradations}
This section extends the result assessing the metrics' degradation sensitivity across domains.
In \Cref{fig:examples_of_degr}, we report visual samples of the degraded images through the perturbations presented in \Cref{subsubsec:degradation_sensitivity}. 
Then, we show the results on the \ours metrics and the baselines in \Cref{app:fig:sensitivity}.
Finally, we propose a per-layer analysis on \ours-DINO focusing on the response of intermediate layers to progressive degradation in \Cref{app:fig:sensitivity_apexdino}.

\begin{figure}
    \centering
    \figura[width=0.5\linewidth]{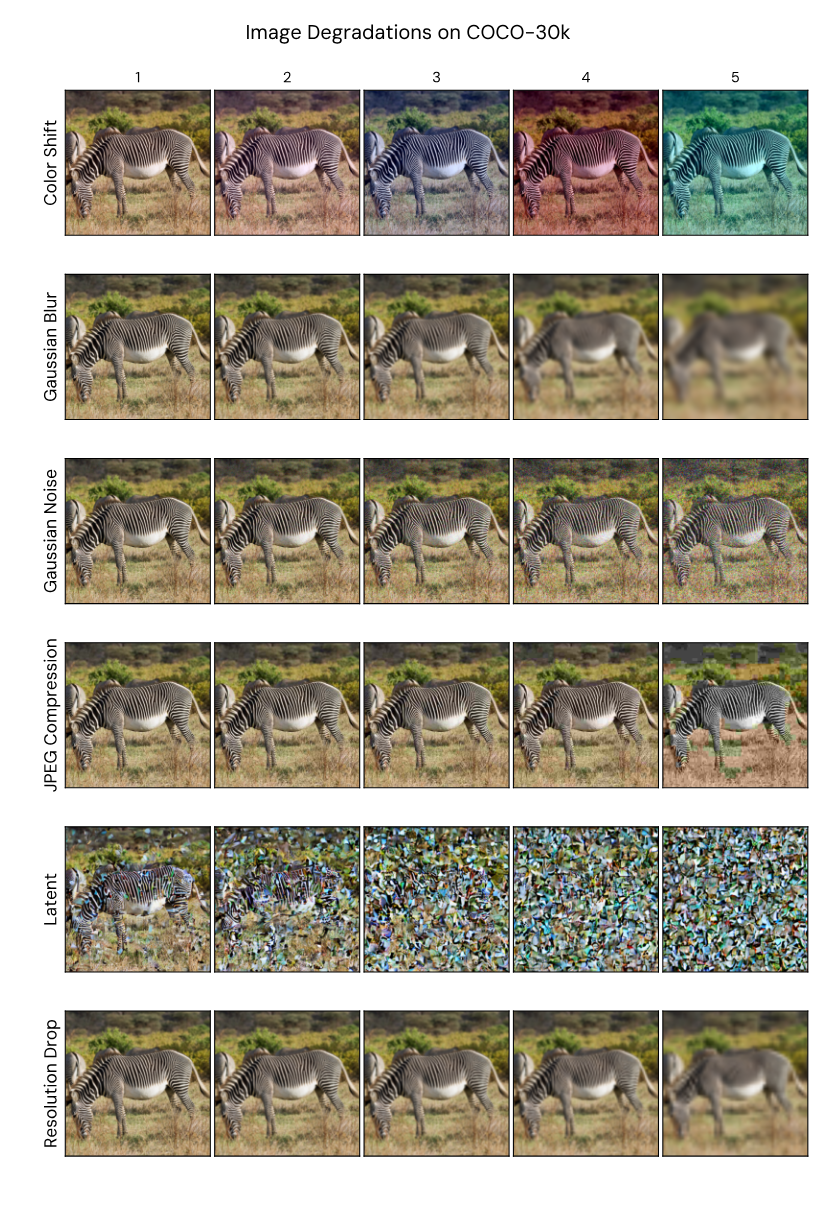}
    \caption{\textbf{Visualization of image degradations on COCO-30k.} Qualitative examples of the six degradation categories across five levels of increasing severity.}
    \label{fig:examples_of_degr}
\end{figure}

\begin{figure}[htbp!]
    \centering
    \figura[width=\textwidth]{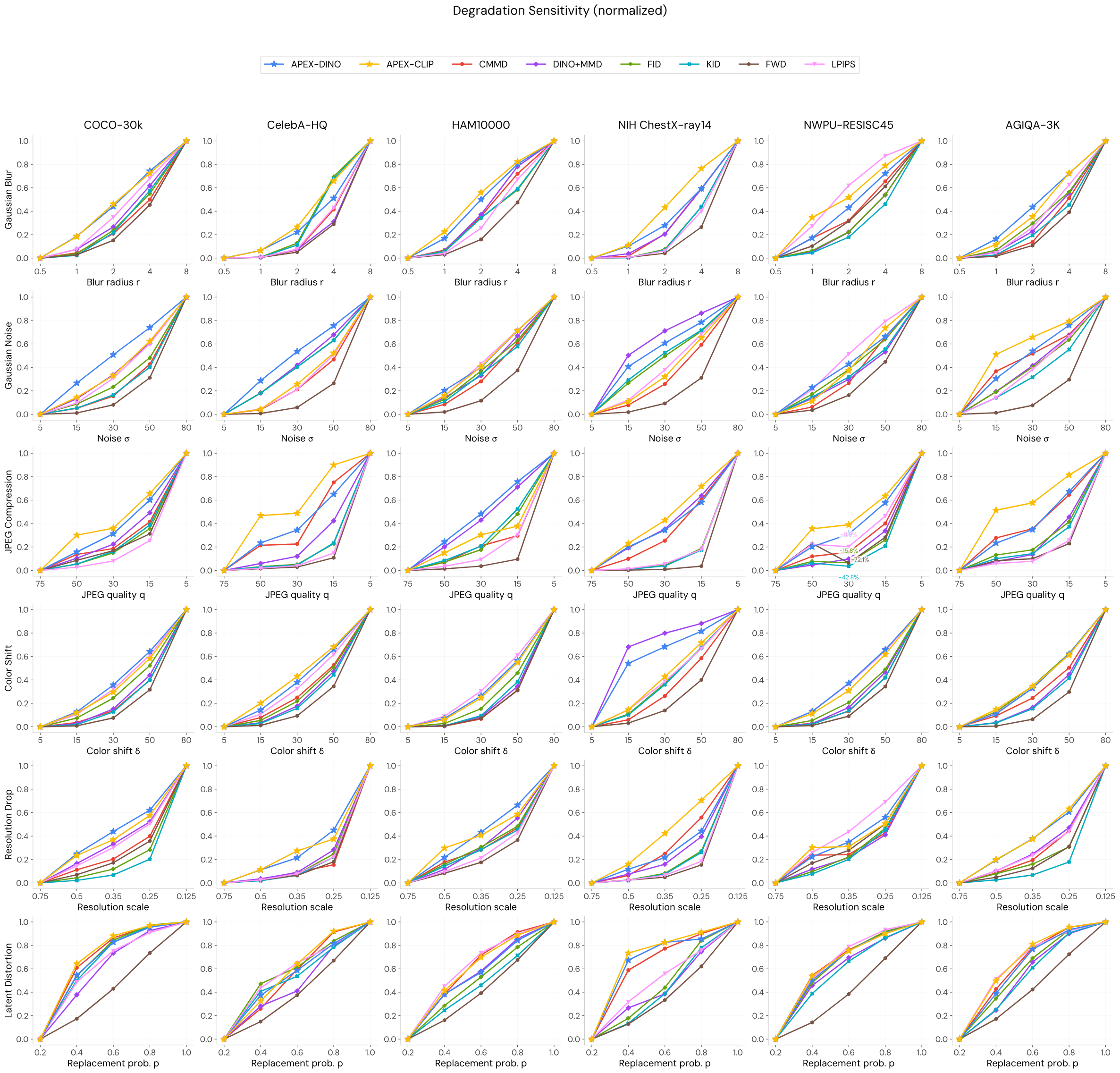}
    \caption{\textbf{Metric sensitivity across degradations and domains.}
    We compare the response of all metrics to progressively stronger perturbations across six degradation types and five evaluation datasets.
    Rows correspond to degradation families, while columns correspond to datasets.
    Each curve is independently min--max normalized to $[0,1]$, and the x-axis reports the degradation parameter used for the corresponding transform.
    \textcolor[HTML]{4285F4}{APEX-DINO} and \textcolor[HTML]{FBBC05}{APEX-CLIP} generally exhibit smooth and consistent responses across domains.
    By contrast, the baselines---\textcolor[HTML]{EA4335}{CMMD}, \textcolor[HTML]{A142F4}{DINO+MMD}, \textcolor[HTML]{65A30D}{FID}, \textcolor[HTML]{00ACC1}{KID}, \textcolor[HTML]{795648}{FWD}, and \textcolor[HTML]{FF95FF}{LPIPS}---show stronger local fluctuations in several settings.
    Percentage annotations mark monotonicity violations, reporting the relative decrease with respect to the previous severity level.
    The heterogeneous behaviour observed across datasets, especially for latent distortions, highlights the presence of degradation--domain interactions.}
    \label{app:fig:sensitivity}
\end{figure}

\begin{figure}[htbp!]
    \centering
    \figura[width=\textwidth]{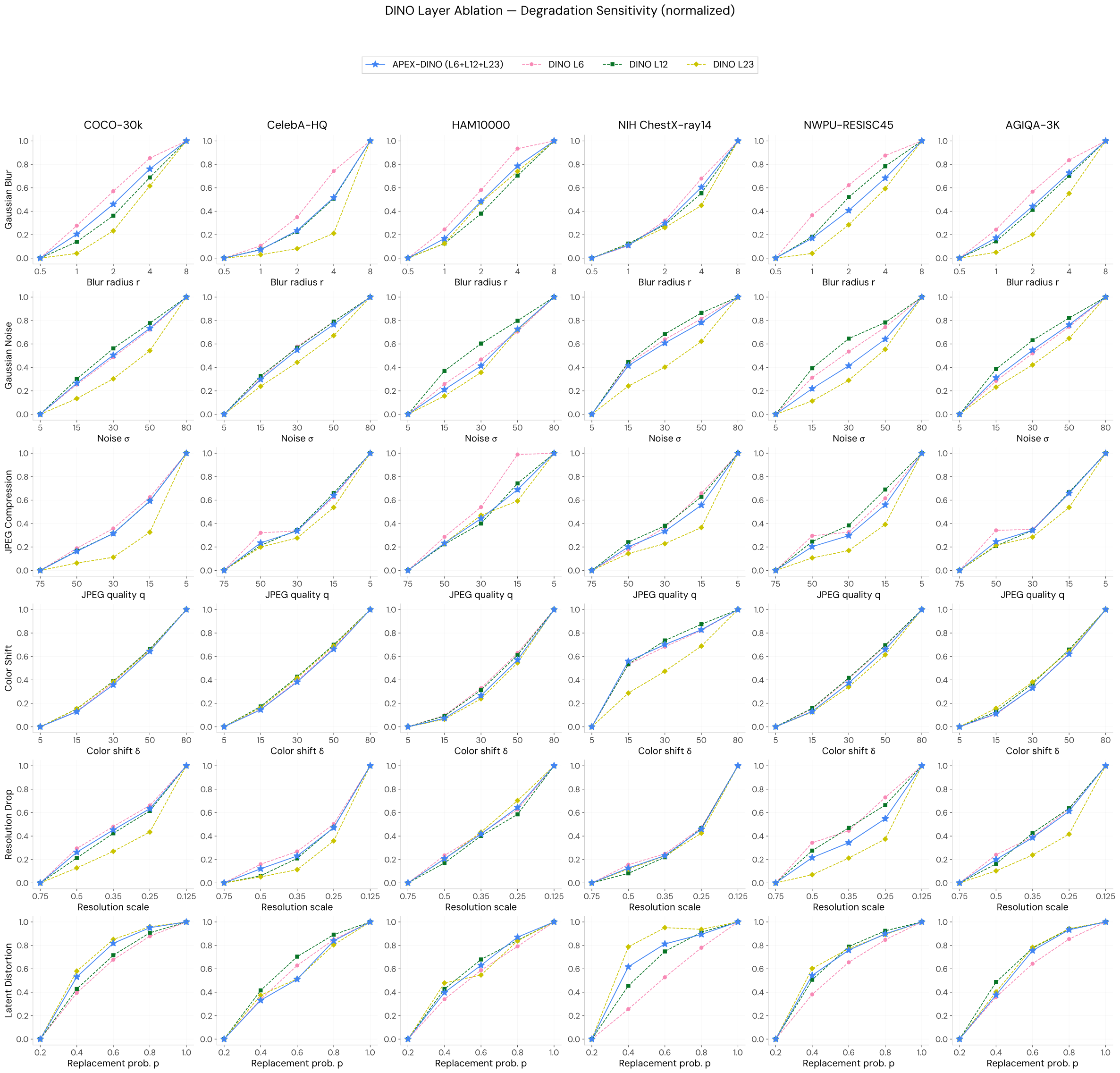}
    \caption{\textbf{Layer-wise sensitivity of APEX-DINO.}
    We report normalized SWD responses computed separately on DINOv2 layers \textcolor[HTML]{F88AB6}{L6}, \textcolor[HTML]{097426}{L12}, and \textcolor[HTML]{CAC400}{L23} and on their concatenation \textcolor[HTML]{4285F4}{APEX-DINO} across datasets and degradations.
    Shallow features \textcolor[HTML]{F88AB6}{L6} are generally most sensitive to low-level corruptions, such as colour shifts, noise, JPEG artifacts, and resolution loss, while deep features \textcolor[HTML]{CAC400}{L23} are more invariant to these perturbations.
    The intermediate layer \textcolor[HTML]{097426}{L12} captures structural changes between local appearance and high-level semantics.
    Latent distortions affect all layers more uniformly, suggesting that they disrupt multiple levels of the visual representation.}
    \label{app:fig:sensitivity_apexdino}
\end{figure}

\section{Image Synthesis and Refinement}
\label{sec:app:img_refinem}
In this section, we provide evidences to evolving image fidelity throughout the generative process in \Cref{app:fig:generation_refinement}. 
We use Stable Diffusion v1.5 to generate samples and extract intermediate image representations.
Additionally, we show the results below.
\begin{figure}[h!]
    \centering
    \figura[width=\textwidth]{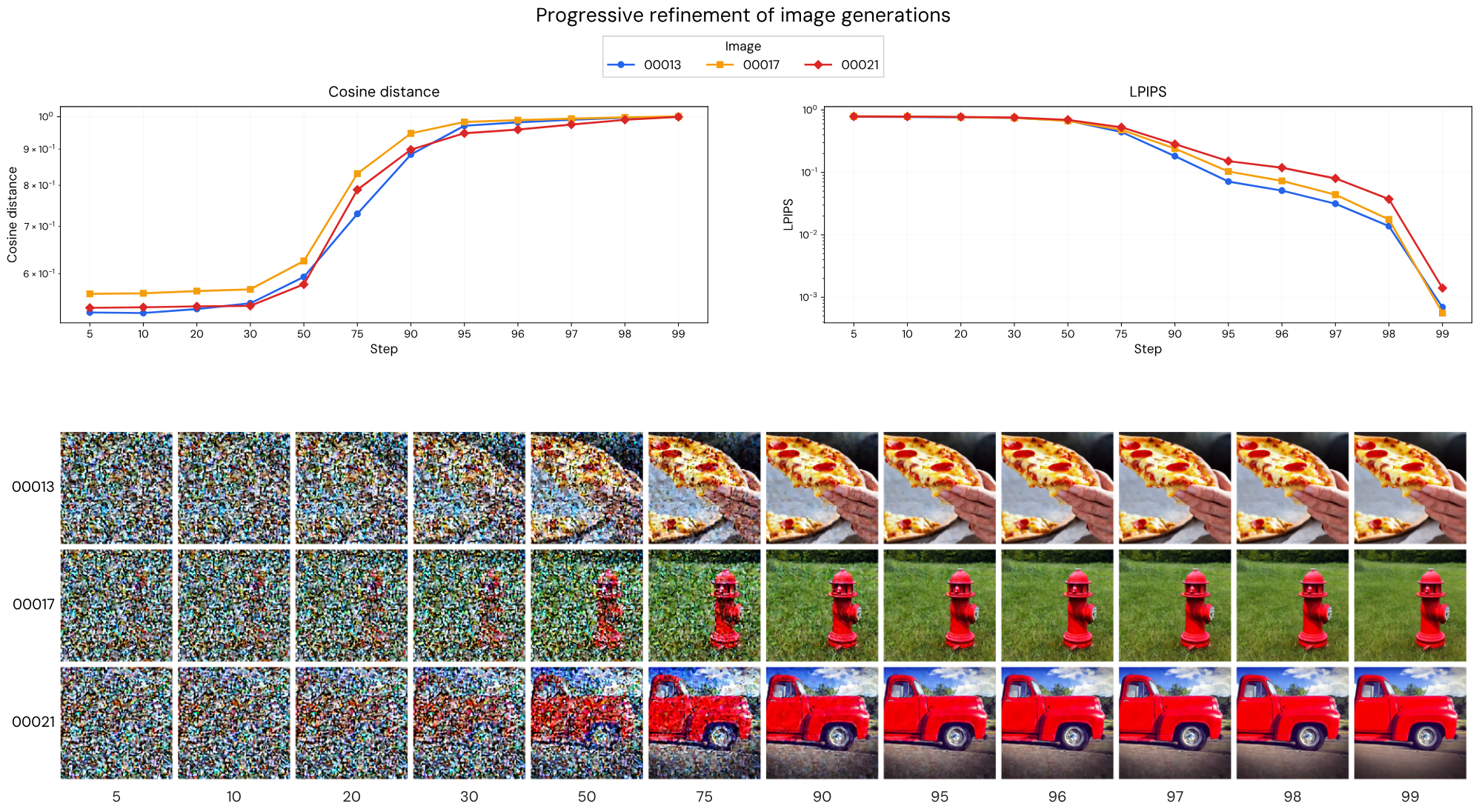}
    \caption{\textbf{Generative models perform refinement in the last timesteps.} 
    Top: Evolution of the cosine distance between CLIP embeddings (semantic score) and LPIPS (perceptual score) across generation steps for selected samples. 
    Bottom: Qualitative results across generative timesteps of Stable Diffusion.
    The semantic CLIP score shows saturation for increasing timesteps, while the perceptual LPIPS score decreases. 
    This empirically demonstrates that the generation process converges to the final image by refining fine details (i.e., color saturation and shadows) in the last steps.}
    \label{app:fig:generation_refinement}
\end{figure}
\begin{figure}[htbp!]
    \centering
    \figura[width=0.6\linewidth]{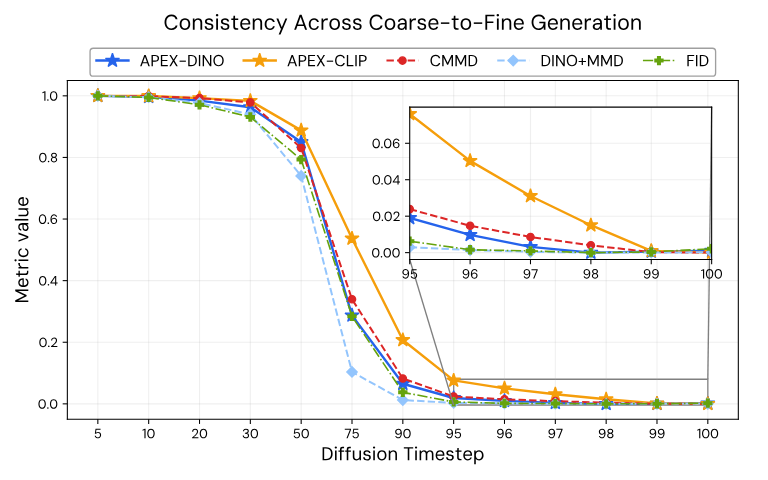}
    \caption{\textbf{Consistency across coarse-to-fine generation.} 
    We evaluate metric responses along the Stable Diffusion denoising process. The metrics correctly exhibit a decreasing trend, reflecting the progressive improvement in sample quality. }
    \label{fig:progressive_generation_metrics}
    \vspace{-\baselineskip}
\end{figure}

\section{Human Perceptual Correlation Study} \label{app:human}
This section expands upon the experimental setup, deployment, analysis and results of the human perceptual evaluation presented in \Cref{subsubsec:human}: a user study on controlled degradations (\Cref{app:human:study,app:human:protocol,app:human:stats}) and an evaluation on the AGIQA-3K benchmark of real generator outputs (\Cref{app:human:agiqa}).

\subsection{Stimuli and Design} \label{app:human:study}
\begin{itemize}[leftmargin=1.5em]
    \item \textbf{Datasets:} To test the domain invariance of the metrics, we select datasets spanning highly diverse visual domains: high-resolution faces (CelebA-HQ), natural images (COCO-30k), medical imaging (ChestX-ray14 and HAM10000), and satellite imagery (RESISC45).
    \item \textbf{Conditions \& Sampling:} For each of the 5 datasets, we construct paired stimuli for five pixel-space degradation families (Gaussian blur, Gaussian noise, JPEG compression, color shift, and resolution drop) at severity levels 1--5, and for latent distortions at replacement probabilities $p \in \{0.20, 0.40, 0.60, 0.80, 1.00\}$.
    \item \textbf{Volume:} We sample $N=70$ source images per dataset and render each of them under every condition (dataset $\times$ degradation $\times$ severity). With $30$ degraded conditions plus the originals, this yields $2{,}170$ files per dataset and $10{,}850$ images in total.
    \item \textbf{Image Handling:} Image pairs are formed by matching filenames between the original and degraded directories, so each degraded image is shown next to its exact source. Both images are scaled with the same aspect-preserving fit to the display area, so original and degraded versions are always shown at identical size.
\end{itemize}

\subsection{Protocol and Deployment} \label{app:human:protocol}
\begin{itemize}[leftmargin=1.5em]
    \item \textbf{Procedure:} Annotators are shown the original image next to the degraded one and asked how different the degraded image looks, on a 1--5 scale (1: almost identical, 5: strongly degraded or difficult to recognize). To calibrate annotators to the scale, each session begins with $3$ practice trials, after which a reference score is displayed (the severity level of the shown condition); practice ratings are discarded. The main phase consists of $36$ trials per session.
    \item \textbf{Sampling and Randomization:} The $36$ trials are split evenly across datasets and, within each dataset, assigned to the degradation--severity conditions that have received the fewest ratings so far, to keep coverage balanced across sessions. The resulting trials are presented in random order.
    \item \textbf{Quality Control:} Responses faster than $1.5$\,s and sessions with fewer than five valid trials are excluded from the analysis; in the collected data no response met these criteria.
    \item \textbf{Infrastructure:} The study is deployed as a single-page React application hosted on Vercel, feeding response data into a Google Sheets backend via a Google Apps Script endpoint. The interface is bilingual (English/Italian) and responsive across desktop, tablet, and mobile platforms. Sessions are tracked via anonymous tokens, and no personally identifiable information is collected.
\end{itemize}

\subsection{Detailed Statistical Analysis} \label{app:human:stats}
For each condition $c = (k, \tau, s)$---denoting dataset $k$, degradation $\tau$, and severity $s$---we compute the Mean Opinion Score (MOS)~\citep{streijl2016mean}:
$$
\mathrm{MOS}(c) = \frac{1}{|\mathcal{A}_c|} \sum_{a \in \mathcal{A}_c} r_a ,
$$
where $r_a \in \{1,\dots,5\}$ is the rating provided by annotator $a$, and $\mathcal{A}_c$ is the set of valid ratings for that condition. The metric degradation signal is $\Delta_d(c) = d\bigl(\tau_s(\mathcal{D}_k),\,\mathcal{D}_k\bigr)$, computed on the full benchmark sets (all real images against all degraded images, at most $10^4$ per side; $L=512$ projections for \ours), not only on the $70$ displayed images. LPIPS is computed on the pairs formed by each real image and its degraded version.

We evaluate the metrics using three criteria:
\begin{itemize}[leftmargin=1.5em]
    \item \textbf{Spearman $\rho$}: rank correlation between $\mathrm{MOS}(c)$ and $\Delta_d(c)$ across all evaluated conditions~\citep{spearman_ro}.
    \item \textbf{Kendall $\tau$}: a secondary concordance measure, robust to tied ranks~\citep{kendall_tau}.
    \item \textbf{Per-dataset $\rho$ variance}: the Spearman correlation computed within each dataset in isolation. The variance of these per-dataset values indicates the metric's domain bias: a metric that tracks human perception on natural images but fails on satellite or medical imagery exhibits high variance, whereas a domain-invariant metric exhibits consistently high $\rho$ with low variance across all 5 datasets.
\end{itemize}
We report two-sided $p$-values (asymptotic for $\rho$; exact for $\tau$ in the absence of ties).

Our analysis aggregates ratings at the condition level. The study collected $900$ ratings in $25$ sessions, covering $142$ of the $150$ conditions ($1$--$19$ ratings per condition, median $4$). Because the sampling spreads ratings across the whole pool, $91.9\%$ of the $828$ rated images received exactly one rating, so individual images cannot be scored reliably and condition-level aggregation is the appropriate unit of analysis ($n=142$).

\subsection{AGIQA-3K} \label{app:human:agiqa}
To complement controlled degradations with outputs of real generative models, we use AGIQA-3K~\citep{li2023agiqa}, which contains $2{,}982$ AI-generated images from ten generator settings (AttnGAN, DALL-E~2, GLIDE, Midjourney with default and reduced steps, Stable Diffusion~1.5 with default settings, reduced steps, and high and low prompt guidance, and SDXL), each annotated with a MOS for perceptual quality and one for text--image alignment.
\begin{itemize}[leftmargin=1.5em]
    \item \textbf{Groups:} Since the metrics compare sets of images, we form groups and correlate each metric with the group's mean MOS: (i) one group per generator setting ($290$--$300$ images); (ii) five equal-size bins of images sorted by quality MOS; (iii) five equal-size bins sorted by alignment MOS (about $596$ images each).
    \item \textbf{Reference:} The real reference is a fixed random subset of COCO-30k (seed 42). A group of $n$ generated images is compared with the first $n$ reference images, so real and generated sets always have the same size and the real subsets are nested.
    \item \textbf{Aggregate score:} For the main table we pool the $20$ groups: generator groups are scored by the mean of their quality and alignment MOS, and each bin by the MOS it was sorted on. We correlate each metric with the negated MOS, so that, as in the user study, a positive correlation means agreement with human judgment. Since the three groupings partition the same images, the $20$ groups are not independent and the pooled $p$-values are optimistic.
    \item \textbf{Per-grouping analysis:} \Cref{tab:metric-correlations-userstudy} reports each grouping and MOS type separately, with exact two-sided permutation $p$-values (all $n!$ rank permutations for $\rho$, exact null distribution for $\tau$); with $n=5$ bins the smallest attainable $p$ is $0.017$.
    \item \textbf{LPIPS:} AGIQA-3K provides no paired reference, so LPIPS compares each generated image with an unrelated real image; we report it for completeness only.
\end{itemize}

\end{document}